\documentclass[letterpaper, 10 pt, conference]{ieeeconf}  

\IEEEoverridecommandlockouts                              
\usepackage[margin=0.75in]{geometry}
\usepackage{amsmath,amsfonts,mathtools,bm}
\usepackage{booktabs,tabularx,array,multirow}
\usepackage{cite}

\usepackage{amsthm}

\usepackage{algorithm}
\usepackage{algorithmic}
\usepackage{xcolor}
\usepackage{graphicx}
\usepackage{tikz}
\usetikzlibrary{positioning,arrows.meta,fit,backgrounds,shapes.misc}
\usepackage[hidelinks]{hyperref}
\usepackage[nameinlink,capitalise]{cleveref}
\usepackage{caption}
\usepackage{dblfloatfix}
\usepackage{cuted}

\newtheorem{theorem}{Theorem}
\newtheorem{proposition}{Proposition}
\newtheorem{corollary}{Corollary}
\newtheorem{remark}{Remark}
\newtheorem{lemma}{Lemma}
\DeclareMathOperator*{\argmin}{arg\,min}
\DeclareMathOperator{\KL}{KL}

\DeclareMathOperator{\diag}{diag}

\newcommand{\X}{\mathcal X}
\newcommand{\U}{\mathcal U}
\newcommand{\T}{\mathcal T}
\newcommand{\D}{\mathcal D}
\newcommand{\E}{\mathbb E}
\newcommand{\R}{\mathbb R}
\newcommand{\sg}{\operatorname{sg}}

\newcommand{\norm}[1]{\left\lVert #1 \right\rVert}

\newif\ifshowchanges
\showchangestrue
\makeatletter
\AtBeginDocument{%
  \let\drift@basefntext\@makefntext
  \let\drift@basefloatreset\@floatboxreset
}
\def\drift@setcolor#1{%
  \color{#1}%
  \long\def\@makefntext##1{{\color{#1}\drift@basefntext{##1}}}%
  \def\@floatboxreset{\drift@basefloatreset\color{#1}}%
  \captionsetup{font+={color=#1}}%
}
\def\drift@pagecounter{page}
\def\drift@snapshotcounters{%
  \let\drift@savedelt\@elt
  \def\drift@restorecounters{}%
  \def\@elt##1{%
    \def\drift@countername{##1}%
    \ifx\drift@countername\drift@pagecounter
    \else
      \edef\drift@restorecounters{%
        \drift@restorecounters
        \global\csname c@##1\endcsname=\number\csname c@##1\endcsname\relax}%
    \fi}%
  \cl@@ckpt
  \let\@elt\drift@savedelt
}
\def\drift@gobblelabel{%
  \@ifnextchar[{\drift@gobblelabel@opt}{\@gobble}}
\def\drift@gobblelabel@opt[#1]#2{}
\DeclareRobustCommand{\old}[1]{%
  \ifshowchanges
    \begingroup
      \drift@snapshotcounters
      \let\drift@savedhref\@currentHref
      \let\drift@savedglobalhref\HyperGlobalCurrentHref
      \let\drift@savedhyperlink\hyper@link
      \let\drift@savedlivelink\hyper@livelink
      \let\label\drift@gobblelabel
      \let\ltx@label\drift@gobblelabel
      \let\label@in@display\drift@gobblelabel
      \@fileswfalse
      \let\protected@write\@gobblethree
      \NoHyper
      \drift@setcolor{red}%
      #1%
      \endNoHyper
      \drift@restorecounters
      \global\let\@currentHref\drift@savedhref
      \global\let\HyperGlobalCurrentHref\drift@savedglobalhref
      \global\let\hyper@link\drift@savedhyperlink
      \global\let\hyper@livelink\drift@savedlivelink
    \endgroup
    \ifhmode\space\fi
  \fi
}
\DeclareRobustCommand{\new}[1]{%
  \begingroup
    \ifshowchanges\drift@setcolor{blue}\else\drift@setcolor{black}\fi
    #1%
  \endgroup
}
\makeatother
\title{\LARGE \bf
Receding-Horizon Control via Drifting Models
}

\author{Daniele Foffano$^\ast$, Alessio Russo$^\ast$ and Alexandre Proutiere
\thanks{$\ast$ Equal contribution}
\thanks{This work was not supported by any organization}
\thanks{D. Foffano and A. Proutiere are with the Division of Decision and Control Systems of the EECS
School at KTH Royal Institute of Technology, Stockholm, Sweden. {\tt\small \{foffano, alepro\}@kth.se}}%
\thanks{A. Russo is at the Faculty of Computing and Data Sciences, Boston University,
        Boston, USA
        {\tt\small arusso2@bu.edu}}%
}

\begin{document}

\maketitle
\thispagestyle{empty}
\pagestyle{empty}

\begin{abstract}
We study the  problem of trajectory optimization in settings where the system dynamics are unknown and the controller cannot simulate candidate trajectories through a model at test time. When an offline dataset of trajectories is available, an agent could  directly learn a trajectory generator by distribution matching. However, this approach only recovers the behavior distribution in the dataset, and does not in general produce a model that minimizes a desired cost criterion. In this work, we propose \emph{Drifting Receding Horizon Control (RHC)}, an offline trajectory optimization framework that combines drifting generative models with receding-horizon planning under unknown dynamics. The goal of Drifting RHC is to learn, from an offline dataset of trajectories, a conditional  distribution over trajectories that is both supported by the data and biased toward optimal plans. We show that the target distribution of Drifting RHC is the unique solution of an objective that trades off optimality with closeness to the offline prior. We further bound the gap between the predicted and the executed cost of a generated plan by its consistency with the dynamics, and introduce a rollout-consistency loss that distills a learned model into the generator at training time only. On a mass-spring-damper, a cart-pole and a planar quadrotor, Drifting RHC attains the lowest regret among generative and model-based planners, outperforms DeePC on the nonlinear systems, and plans $9$ to $83$ times faster than iLQR, model-based and diffusion-based planners.
\end{abstract}

\section{Introduction}\label{sec:introduction}

Trajectory optimization lies at the heart of control and robot autonomy. In many settings, however, computing an accurate control law through trajectory optimization is itself difficult because it requires a model of the environment that is sufficiently faithful for planning. While data-driven modeling \cite{willems2005note,tang2022data,coulson2019data,russo2023Tubebased} and system identification \cite{ljung1998system} provide natural alternatives, in the presence of noise or nonlinearities, it may be difficult to construct a  model of the system that is accurate enough to compute a control law. The system dynamics may be unknown or hard to identify precisely. In such cases, the agent must rely on an offline dataset of previously collected trajectories, often generated by heterogeneous and possibly suboptimal controllers, to learn an optimal action plan.

Reinforcement learning \cite{sutton2018reinforcement} is a natural alternative in this setting. In particular, model-based RL \cite{polydoros2017survey,argenson2021ModelBased} seeks to learn a  model of the system from data and then use that model for planning or policy improvement \cite{janner2019trust,yu2020mopo,kidambi2020morel}. This can be effective when the learned model is sufficiently accurate in the regions visited by the resulting controller. However, the quality of the control law then depends on the quality of the learned model over long horizons, and even small prediction errors may accumulate during rollouts and distort the optimization process \cite{foffano2025Adversariala,janner2019trust,yu2020mopo}. This issue is particularly acute in offline settings, where the model must be learned from a fixed dataset and cannot be corrected through further interactions \cite{kidambi2020morel,prudencio2023survey}.

A different alternative is to bypass planning through an explicit model at test time, and instead learn a conditional generator that proposes full trajectory plans directly from data \cite{NEURIPS2021_099fe6b0,pmlr-v162-janner22a}. In this setting, one does not optimize through identified transition dynamics at every control step. Rather, one learns a distribution over trajectories themselves, conditioned on the current state. This perspective is attractive for two reasons: (i)  it amortizes computation  so that planning at test time reduces to generating and scoring candidate trajectories; (ii) it shifts the learning problem from system identification to trajectory generation \cite{NEURIPS2021_099fe6b0,chen2021decision,pmlr-v162-janner22a}, so that the planner cannot exploit the errors of a model by optimizing through it. In principle, such a model can be embedded inside a receding-horizon control strategy by repeatedly sampling candidate trajectories from the current state and executing the control action from the one with smallest cost.

Yet, this approach introduces a fundamental mismatch between modeling and optimization. A generator trained only by distribution matching will reproduce the behavior contained in the offline dataset, namely what the data-collecting controllers tended to do, but not necessarily what is optimal according to some cost criterion. The challenge is therefore to retain the computational advantages of direct trajectory generation while biasing the generated plans toward optimal solutions.

Diffusion-based planners address this challenge by combining trajectory generation with reward or cost guidance \cite{pmlr-v162-janner22a}, and have shown that generative models can be effective tools for planning. Their main drawback is computational: they require multiple denoising steps at inference time, which makes them difficult to deploy in real-time control scenarios where the planner must be queried repeatedly. Drifting models provide a complementary alternative \cite{deng2026Generative}. They replace iterative denoising with a one-step pushforward generator trained through an attraction-repulsion field, making them attractive for fast receding-horizon planning. However, if applied directly to offline data, they still converge to the empirical behavior distribution, and therefore do not in general solve the trajectory optimization problem of interest.

In this work, we propose \emph{Drifting RHC}, an offline trajectory optimization framework that combines drifting generative models with receding-horizon planning under unknown dynamics. The key idea is to modify the positive drift field so that the learned generator is trained toward an optimal target distribution rather than toward the raw dataset distribution. More precisely, we use an exponentially tilted offline prior, which increases the influence of optimal trajectories while preserving regularization toward the local support of the data. The learned generator is then used inside a best-of-$M$ receding-horizon planner: given the current state and a cost query, it proposes several candidate trajectories, their known quadratic cost is evaluated, and the first control of the best candidate is executed.

Since the candidates are ranked through their own predicted states, the planner may also select a plan that the system cannot follow. We therefore distill a one-step model, learned from the same dataset, into the generator through a rollout-consistency loss, used at training time only. The contributions of this paper are threefold:
\begin{itemize}
    \item We formulate Drifting RHC, a cost-aware one-step generative approach to offline trajectory optimization, whose target optimizes expected cost with KL regularization to the offline prior.
    \item We bound the gap between predicted and executed cost of a generated plan by its consistency with a learned model, and introduce a rollout-consistency loss that controls it while keeping inference model-free.
    \item We validate our approach on three systems against six baselines, showing near-optimal plans and a notable speed increase over diffusion-based, model-based and data-driven predictive control.
\end{itemize}

\begin{figure*}[h]
    \centering
    \includegraphics[width=0.85\textwidth]{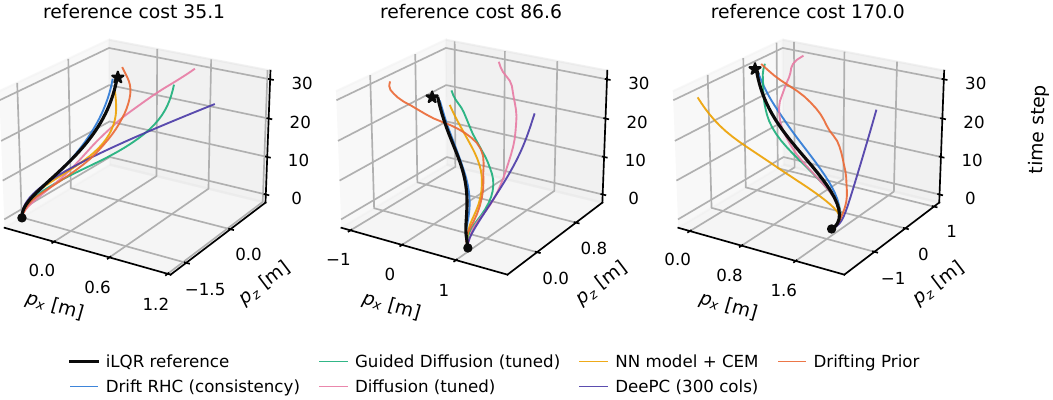}
    \caption{Planar quadrotor: closed-loop trajectories in the $(p_x,p_z)$ plane against time for three test episodes of increasing cost, compared with the iLQR reference (black). The star marks the hover target at the origin, the circle the initial state.}
    \label{fig:quadrotor}
\end{figure*}

\section{Related Work}\label{sec:related_work}

\paragraph{Trajectory modeling for offline decision-making}
A closely related line of work views offline decision-making as a sequence or trajectory modeling problem. Decision Transformer casts offline reinforcement learning as return-conditioned sequence generation \cite{chen2021decision,emmons2021rvs}, while Trajectory Transformer models trajectories autoregressively and performs planning in the learned sequence space \cite{NEURIPS2021_099fe6b0}.  More broadly, conditional generative modeling has also been proposed as a general framework for offline decision-making \cite{ajay2023Conditional}.  These works share the same high-level motivation as ours: replace explicit dynamic programming or repeated model rollouts with a learned model over trajectories. Our setting, however, is different in two key respects. First, we target one-step trajectory generation rather than autoregressive generation. Second, instead of conditioning on desired return or using search over a sequence model, we bias the learned distribution by changing the positive distribution in the drifting objective itself.

\paragraph{Diffusion-based planning and control}
Diffusion models \cite{ho2020denoising} have recently become a standard tool for generative planning. Diffuser formulates planning as conditional trajectory denoising and includes rewards, constraints, or goal information through test-time guidance \cite{pmlr-v162-janner22a}. Related diffusion-based methods have also been used as policy classes for offline reinforcement learning and robot control \cite{wang2023Diffusion,chi2024Diffusion}. These methods form the most natural baseline family for our work. The main difference is computational: diffusion planners rely on multiple denoising steps at inference time, whereas Drifting RHC aims to learn a one-step proposal model that can be queried repeatedly inside a receding-horizon loop. Moreover, a plan generated jointly over states and actions need not be dynamically consistent: the Decision Diffuser recovers the actions from generated states with an inverse-dynamics model \cite{ajay2023Conditional}, and Diffusion MPC scores action proposals with a learned dynamics model at planning time \cite{zhou2025Diffusion}, while Drifting RHC enforces consistency at training time only.

\paragraph{Offline model-based planning and learning-augmented MPC}
Another nearby literature studies offline planning through learned dynamics models. Model-based offline reinforcement learning and planning methods such as MBOP, MOPO, MOReL, and MOPP learn a surrogate model from static data and then optimize through that model at test time \cite{argenson2021ModelBased,yu2020mopo,kidambi2020morel,zhan2022ModelBased}, often with sampling-based optimizers such as the cross-entropy method (CEM) \cite{rubinstein1999Crossentropy,chua2018Deep}. Closely related are learning-augmented MPC approaches, which use learned proposal distributions or sequence models to warm-start or regularize online planning \cite{sacks2022Learning,celestini2024transformer,russo2023Tubebased,10155796}. Drifting RHC is related to this line of work in that it is also used inside a receding-horizon controller, but it differs in a fundamental way: it does not roll candidate trajectories through a learned model or solve a trajectory optimization problem online. A learned model only enters the training loss of the generator, which amortizes it into a one-step proposal mechanism, similarly to policies trained inside a learned model and deployed without planning \cite{janner2019trust,hafner2020Dream}.

Data-enabled Predictive Control (DeePC) instead uses Willems' fundamental lemma \cite{willems2005note} to solve a receding-horizon problem directly from a Hankel matrix of data \cite{coulson2019data,dorfler2023Bridging}: it recovers the optimal controller of a linear system, but solves an optimization problem at every step, and only provides a local linear description of nonlinear systems.

\paragraph{Drifting models and regularized control.}
Our method builds directly on Drifting Models, which replace iterative denoising of diffusion models with a single pushforward map \cite{deng2026Generative}. Standard drifting, however, matches the data distribution and therefore recovers the offline behavior prior rather than an objective-aware planning distribution. Drifting RHC addresses exactly this limitation, so that the generator is biased towards optimal trajectories. This idea is also closely related to the broader control-as-inference view of regularized optimal control, where optimal control distributions arise by exponentially tilting a prior with task-dependent rewards or costs \cite{levine2018Reinforcement}. 
\section{Background and Problem Definition}\label{sec:background}

\subsection{Setting}

\paragraph { Model} In the following we consider a discrete-time control system
$
\mathcal M = (\X, \U, f, \rho_0, H),
$
where $\X \subseteq \R^{d_x}$ is the state space, $\U \subseteq \R^{d_u}$ is the control space, $f$ is the unknown transition law over the next state, $\rho_0$ is the initial-state distribution, and $H \in \mathbb N$ is the planning horizon. Therefore, at each timestep the state evolves according to $x_{t+1}\sim f(\cdot\mid x_t,u_t)$, where $u_t$ is the control action at timestep $t$, and the initial state is $x_0\sim \rho_0$.

In the following, we write an $H$-step trajectory as
\[
\tau = (x_0,u_0,x_1,u_1,\ldots,x_{H-1},u_{H-1},x_H), \qquad x_0=x,
\]
where $x$ is the initial state, and by $\tau_t=(x_0,u_0,\dots, x_t,u_t)$ its truncation to the first $t$ transitions. 
We also denote by $\T_H(x)$ the set of all such $H$-step trajectories starting at $x$, and by $\T_H=\bigcup_{x\in\X}\T_H(x)$ the corresponding global trajectory space.

\paragraph{ Trajectory cost}
In the following, we associate to each trajectory $\tau$ a cost. Specifically, the finite-horizon cost associated with a trajectory $\tau\in\T_H(x)$ is
\begin{equation}
J_x(\tau;\omega)
\coloneqq
\sum_{t=0}^{H-1}\bigl(x_t^\top Q(\omega)x_t + u_t^\top R(\omega)u_t\bigr) + x_H^\top Q(\omega)x_H.
\label{eq:cost}
\end{equation}
where $Q(\omega)\in \mathbb{R}^{d_x\times d_x}$ and $R(\omega)\in \mathbb{R}^{d_u\times d_u}$ are semi-definite positive matrices  parametrized by  a cost parameter $\omega=(q,r)\in \Omega$ in a compact set $\Omega$. For simplicity, without loss of generality in the following we assume $Q(\omega)={\rm diag}(q)$ and $R(\omega)={\rm diag}(r)$ and
$\Omega = [q_{\min},q_{\max}]^{d_x}\times[r_{\min},r_{\max}]^{d_u}$, with non-negative bounds $q_{\rm min}\geq 0, q_{\rm max}\geq 0$.

\subsection{Problem Definition}
We seek a data-driven receding-horizon control law that minimizes the expected cost $J_x(\cdot;\omega)$ at any starting state $x$ and cost parameter $\omega$. Unlike previous data-driven methods, we take a generative approach, and learn a law $\mu_{x,\omega}$ over the set of trajectories $\T_H(x)$ that minimizes the expected trajectory cost.

\paragraph{Objective} In the following, for any fixed $(x,\omega)\in \X\times \Omega$ we consider  the following optimization problem
\begin{equation}\label{eq:planning_problem}
\begin{aligned}
&\min_{\mu\in \Delta(\T_H(x))} 
 \E_{\tau \sim \mu}\left[J_x(\tau;\omega)\right] \\
 &\text{s.t. }\;
\mu\left({\rm d}x_{t+1}\mid \tau_t\right)
= f\left({\rm d}x_{t+1}\mid x_t,u_t\right)
\quad \text{a.s.} \;\;\forall t< H,
\end{aligned}
\end{equation}
where $\Delta(\T_H(x))$ denotes the set of probability measures over $H$-step trajectories starting at $x$.

\paragraph{Control law} A minimizer $\mu_{x,\omega}^\star$ of \cref{eq:planning_problem} defines a receding-horizon control law: at every timestep, after observing the state $x$, the agent samples a trajectory $\tau \sim \mu_{x,\omega}^\star$, executes its first control action $u_0$, and repeats the procedure at the next state.

Importantly, we require the control law to be \emph{model-free at inference}: candidate controls are not scored by simulating them forward through a model, which may only be learned and used offline, during training. This motivates learning a full distribution over $H$-step trajectories, and not only actions, which must minimize the expected cost while keeping its predicted states faithful to the transition law.

\paragraph{Offline trajectories} Lastly, we assume the agent has access to an offline dataset of trajectories $\D = \{\tau_i\}_{i=1}^N$, where each $\tau_i$ is a horizon-$H$ trajectory segment generated by some control law. In the following, we denote by $
 P_\D=\frac{1}{|\D|}\sum_{i=1}^{|\D|}\delta_{\tau_i}
$
 the  empirical distribution of trajectories in $\D$. The transitions $(x_t,u_t,x_{t+1})$ contained in $\D$ will also be used in \cref{sec:consistency} to fit a model of the dynamics at training time.

\subsection{Drifting Models}

Drifting models are one-step generative models that learn by iteratively transporting samples toward a desired target distribution $p$ \cite{deng2026Generative}. The goal of these models is to train a generator so that it produces a sample distributed according to $p$ in a single forward pass.

Mathematically, given some noise $\varepsilon \sim {\cal N}(0,I)$, a generator produces a sample
$
z = G_\theta(\varepsilon),
$
which induces a conditional distribution $q_\theta$. Then, at training time, the evolution of a sample $z$ is  governed by the following equation
$
z_{k+1}=z_k+ V_{p,q_\theta}(z_k)
$
where $V_{p,q_\theta}$ is a \emph{drift field} that quantifies the shift. Following \cite{deng2026Generative}, we obtain that  $V_{p,q_\theta}=0$ when $p=q_\theta$ (it is also  possible to give some sufficient conditions under which $V_{p,q_\theta}\approx 0$ implies $q_\theta\approx p$). Therefore,  the objective is to train $G_\theta$ so that $\E[\|V_{p,q_\theta}\|]\approx 0$.

To define a drift field $V_{p,q_\theta}$, the main idea is to ensure that it moves generated samples toward a chosen positive distribution $p$ while repelling them from the current model distribution $q_\theta$. To that aim, we define the drift field as 
\begin{equation}
V_{p,q_\theta}(z) = V_p^+(z) - V_{q_\theta}^-(z).
\label{eq:drifting_full_field_bg}
\end{equation}
where $V_p^+$ is the positive mean-shift field and $V_q^-(z)$ is the negative mean-shift field. To define these fields, we introduce the Gaussian kernel $k(z,z') = \exp(-\|z-z'\|_2^2/T)$, which measures local similarity, with temperature $T > 0$. Then, the positive mean-shift field is defined as
\begin{equation}
V_p^+(z)
=
\frac{\E_{z^+ \sim p}
\left[k\left(z,z^+\right)\left(z^+ - z\right)\right]}
{\E_{z^+ \sim p}
\left[k\left(z,z^+\right)\right]},
\label{eq:drifting_positive_field_bg}
\end{equation}
and the negative mean-shift field $V_{q_\theta}^-(z)$ is defined analogously, with $z^-\sim q_\theta$ in place of $z^+\sim p$.
Training is performed with a fixed-point objective: if $z_\theta = G_\theta(\varepsilon)$, then the model is updated so that $z$ moves toward its drifted target,
\begin{equation}
\mathcal L_{\mathrm{drift}}(\theta)
=
\E_{z\sim q_\theta}
\left[
\left\|
z
-
\sg\left(z + V_{p,q_\theta}(z)\right)
\right\|_2^2
\right],
\label{eq:drifting_loss_bg}
\end{equation}
where $\sg(\cdot)$ denotes stop-gradient.
\section{Method}\label{sec:method}
Drifting RHC builds on drifting models \cite{deng2026Generative} to approximately solve \cref{eq:planning_problem} for any $(x,\omega)$ from the offline dataset $\D$. It learns a conditional generator that maps noise, the current state, and the cost parameter to trajectories that are both supported by the data and skewed toward low-cost plans, and uses it as a proposal mechanism inside a best-of-$M$ receding-horizon planner. Since applying drifting directly would only yield $\mu \approx P_\D$, does not minimize \cref{eq:planning_problem}, we change the positive mean-shift field $V_{P_\D}^+$ to steer the generator toward low-cost trajectories.

\subsection{Conditionally Shifted Drift Fields}\label{sec:shifted}
In the following we denote the conditioning on starting state and cost parameter by $c=(x_0,\omega)\in\X\times\Omega$, and  define  a conditional generator as a parametrized model  $G_\theta:\mathbb{R}^{d_\epsilon}\times \X\times\Omega\to \mathbb{R}^{H\cdot (d_x+d_u)}$ that takes as input the conditioning $c$ and noise $\epsilon\sim p_\epsilon$ (with $p_\epsilon={\cal N}(0,I)$). The generator maps noise to samples  $z=G_\theta(\epsilon,c)$, which encode an $H$-step trajectory starting at $x_0$ through its controls and its state offsets $x_t-x_0$: hence, the generator induces the conditional pushforward distribution $q_\theta(\cdot\mid c) = [G_\theta(\cdot,c)]_\# p_\epsilon$.

\subsubsection{Target Distribution} The goal of generative modeling is to train $G_\theta$ so that it induces some desired target distribution. Assume we have some prior distribution $p_0(\cdot\mid x_0)$ over trajectories  with starting state $x_0$ and  satisfying the transition law $f$.
The key idea is to replace this prior with a cost-aware positive distribution. For an inverse temperature $\beta>0$, we define the tilted target $p_\beta({\rm d}\tau\mid c)\propto\exp(-\beta J_{x_0}(\tau;\omega))\,p_0({\rm d}\tau\mid x_0)$, whose weight $\exp(-\beta J_{x_0}(\tau;\omega))$ down-weights trajectories with larger cost and promotes trajectories with lower cost.

The distribution $p_\beta$ solves the unconstrained trajectory-law problem in \cref{eq:freeenergy}, as stated by the Gibbs variational principle \cite{donsker1975Asymptotic,levine2018Reinforcement}. This problem is obtained by adding a KL penalty relative to $p_0$, restricting attention to laws $p\ll p_0$, and dropping the transition constraints in \cref{eq:planning_problem}. It is therefore a relaxation of the prior-restricted, KL-regularized control problem. We restate the variational characterization below, with its proof in \cref{app:proofs}.

\begin{theorem}[Variational characterization of the tilted distribution]
\label{thm:variational}
Fix $c=(x_0,\omega)$ and $\beta>0$. Let $p_0(\cdot\mid x_0)$ be a reference distribution over  trajectories and define
\begin{equation}
p_\beta(d\tau\mid c)
=
\frac{e^{-\beta J_{x_0}(\tau;\omega)}}{Z_\beta(c)}p_0(d\tau\mid x_0),
\label{eq:tiltmeasure}
\end{equation}
where
$
Z_\beta(c)=\int e^{-\beta J_{x_0}(\tau;\omega)}p_0(d\tau\mid x_0).
$
Then $p_\beta(\cdot\mid c)$ is the unique minimizer of
\begin{equation}
\begin{split}
\min_{p\ll p_0(\cdot\mid x_0)} \E_{\tau\sim p}\left[J_{x_0}(\tau;\omega)\right] + \frac{1}{\beta}\KL\left(p(\cdot\mid c)\|p_0(\cdot\mid x_0)\right).
\end{split}
\label{eq:freeenergy}
\end{equation}
\end{theorem}

\Cref{thm:variational} shows that $p_\beta$ trades off two competing objectives, minimizing the control cost and remaining close to the offline trajectory prior, with $\beta$ setting how aggressively it shifts away from $p_0$ toward low-cost trajectories. In practice, we increase $\beta$ during training: small values recover a behavior-like target that stabilizes optimization, and larger values gradually sharpen it toward low-cost plans.

\begin{remark}[Dynamics feasibility of the tilted target] For deterministic dynamics, if $p_0(\cdot\mid x_0)$ is supported on trajectories satisfying $x_{t+1}=f(x_t,u_t)$, then every law $p\ll p_0(\cdot\mid x_0)$ is supported on such trajectories. In this case, the transition constraints are automatic, and \cref{eq:freeenergy} coincides with the prior-restricted, KL-regularized control problem. For stochastic dynamics, absolute continuity does not generally preserve the conditional transition kernel. Consequently, $p_\beta$ need not be dynamically realizable under $f$. We use it as a cost-aware proposal target. The executed-cost analysis in \cref{sec:consistency} is restricted to deterministic dynamics. Even in the deterministic case, the minimizer at finite $\beta$ need not minimize the original unregularized objective.
\end{remark}

\subsubsection{Tilting Lemma} We now show how $p_\beta$ can be learned with drifting models. Given $x_0$ and $\tau\in \T_H(x_0)$, let $V_{p_0}^+(\tau;x_0)$ and $V_{p_\beta}^+(\tau;x_0)$ denote the positive mean-shift field in \cref{eq:drifting_positive_field_bg}, with $p=p_0(\cdot\mid x_0)$ and $p=p_\beta(\cdot\mid x_0)$, respectively, computed on trajectories with the same Gaussian kernel $k$. Since $p_\beta\propto \exp(-\beta J_{x_0})p_0$, it is sufficient to tilt the mean-shift field of $p_0$.
\begin{lemma}[Tilting]\label{lem:weighted_positives_tilt}
Fix $c=(x_0,\omega)$ and
let $p_0(\cdot|x_0)$ be a reference distribution over trajectories. Define the weight $w_\beta(\tau;c)=\exp(-\beta J_{x_0}(\tau;\omega))$.
Define the weighted mean-shifted operator:
\[
V_{p_0}^+(\tau;\beta,c)\coloneqq \frac{\mathbb{E}_{\tau'\sim p_0(\cdot\mid x_0)}\left[w_\beta(\tau';c)k(\tau,\tau')(\tau'-\tau)\right]}{\mathbb{E}_{\tau'\sim p_0(\cdot\mid x_0)}\left[w_\beta(\tau';c)k(\tau,\tau')\right]}.
\]

We have that
\[
V_{p_\beta}^+(\tau;x_0)=V_{p_0}^+(\tau;\beta,c)\quad  p_0-a.e.
\]
\end{lemma}
This result, proved in \cref{app:proofs}, is what makes Drifting RHC feasible to implement: the algorithm only needs relative importance weights of the form $e^{-\beta J}$, not samples from the unknown normalized target distribution $p_\beta$.

\subsubsection{Drift Loss} We now define the drift loss used to train $G_\theta$. From the offline data $\D$, we compute the empirical prior
\begin{equation}
\hat p_0(\cdot\mid x_0)
\coloneqq
\sum_{i=1}^N \alpha_i(x_0)\,\delta_{\tau_i}(\cdot),
\label{eq:localprior}
\end{equation}
with normalized weights $\alpha_i(x_0) \propto k_x(x_0,x_i)\,\mathbf 1\{i\in\mathcal N_K(x_0)\}$, where $x_i$ is the initial state of trajectory $\tau_i$. The kernel $k_x:\mathbb{R}^{d_x}\times \mathbb{R}^{d_x}\to [0,\infty)$  measures the similarity between $x_0$ and $x_i$   and $\mathcal N_K(x_0)$ is the set of $K$ nearest neighbors retrieved from the offline dataset $\D$ defined according to some distance. Intuitively, \eqref{eq:localprior} focuses the prior generator on trajectories that are compatible with $x_0$, and with $k_x(x_0,x_i)=\mathbf{1}_{\{x_0=x_i\}}$ we retrieve the true empirical prior.

Then, for a given $c=(x_0,\omega)$, we define the  empirical drift loss  by the following fixed-point objective
\begin{equation}
\mathcal L_{\rm drift}(\theta;c)
=
\frac{1}{2}\E\left[\norm{\tau-\sg\left(\tau+\widehat V_{\hat p_0, q_\theta}(\tau;\beta,c)\right)}_2^2\right],
\label{eq:loss}
\end{equation}
where the empirical drift field $\widehat V_{\hat p_0, q_\theta}(\tau;\beta,c)$ for $\tau\sim q_\theta(\cdot\mid c)$ is defined as
\[
\widehat V_{\hat p_0, q_\theta}(\tau;\beta,c)= \widehat V_{\hat p_0}^+(\tau;\beta,c)- \widehat V_{q_\theta}^-(\tau;c),
\]
with positive drift field $\widehat V_{\hat p_0}^+(\tau;\beta,c)$ and negative drift field  $\widehat V_{q_\theta}^-(\tau;c)$. These fields are defined as follows
\begin{align*}
    \widehat V_{\hat p_0}^+(\tau;\beta,c)&\coloneqq {\E}_{{\cal B}^+,\tau^+}\left[\tilde k_{{\cal B,\beta}}^+(\tau,\tau^+;c) (\tau^+-\tau) \right],\\
    \widehat V_{q_\theta}^-(\tau;c)&\coloneqq {\E}_{{\cal B}^-,\tau^-}\left[ \tilde k_{{\cal B}}^-(\tau,\tau^-;c)(\tau^- -\tau) \right]
\end{align*}
with $\tau^+\sim \hat p_0(\cdot\mid x_0),\; \tau^-\sim q_\theta(\cdot\mid c)$. The batch  ${\cal B}=\{{\cal B}^+,{\cal B}^-\}$ is used to empirically approximate the weights, and it contains positive samples ${\cal B}^+=(\tau_i^+)_{i=1}^K$, $\tau_i^+\sim \hat p_0(\cdot\mid x_0)$, and negative samples ${\cal B}^-=(\tau_i^-)_{i=1}^M$, $\tau_i^-\sim q_\theta(\cdot\mid c)$.
Then, the normalized  positive weights $\tilde k_{\cal B}^+(\tau,\tau^+)$  for the positive field is defined as
\begin{equation}\tilde k_{{\cal B},\beta}^+(\tau,\tau^+;c)=
\frac{e^{-\beta \widetilde J_i(\omega)}\,k(\tau,\tau^+)}{\sum_{\tau_j\in {\cal B}^+} e^{-\beta \widetilde J_j(\omega)}\,k(\tau,\tau_j)}.
\label{eq:weights_pos}
\end{equation}
where we used \cref{lem:weighted_positives_tilt} and  $\widetilde J_i(\omega)$ denotes the cost of the $i$-th trajectory in ${\cal B}^+$, computed using the cost parameter $\omega$.  Lastly, the normalized weight for the negative field is
\begin{equation}\tilde k_{{\cal B}}^-(\tau,\tau^-;c)=
\frac{k(\tau,\tau^-)}{\sum_{\tau_j\in {\cal B}^-\setminus\{\tau\}} k(\tau,\tau_j)}.
\label{eq:weights_neg}
\end{equation}
Hence, \Cref{eq:loss} moves each generated sample toward a local, cost-aware mean shift of the offline data while repelling it from the current model distribution.

\paragraph{Cost normalization} Since the new costs may span several orders of magnitude, in practice we replace $\widetilde J_i(\omega)$ in \cref{eq:weights_pos} with the normalized
\begin{equation}
\bar J_i(\omega)=\mathrm{clip}\left(\frac{\log \widetilde J_i(\omega)-m_{{\cal B}^+}}{s_{{\cal B}^+}},-\bar c,\bar c\right),
\label{eq:logz}
\end{equation}
where $m_{{\cal B}^+}$ and $s_{{\cal B}^+}$ are the mean and standard deviation of $\{\log \widetilde J_j(\omega)\}_{\tau_j\in{\cal B}^+}$, and $\bar c>0$. Without clipping, the weights are proportional to $\widetilde J_i(\omega)^{-\beta/s_{{\cal B}^+}}$, so that, by \cref{thm:variational,lem:weighted_positives_tilt}, they target \cref{eq:freeenergy} with cost $\log J_{x_0}$ and inverse temperature $\beta/s_{{\cal B}^+}$, which is invariant to a rescaling of the cost.

\subsection{Dynamics-Consistent Drift Fields}\label{sec:consistency}
When generating trajectory plans using a drifting model, the predicted future states may not be consistent with the planned actions. The drifting loss pushes the generated samples toward real observed trajectories, while the exponential tilting (and the best-of-M selection introduced in the next section) may favor candidate states with a small predicted cost. Here we introduce a consistency loss term that will help enforcing consistency in the single generated steps. We assume for simplicity deterministic dynamics $x_{t+1}=f(x_t,u_t)$, and write a generated trajectory as $\tau=(x_0,u_0,\hat x_1,\ldots,u_{H-1},\hat x_H)$, with predicted states $(\hat x_t)_{t=1}^H$. The analysis extends to stochastic dynamics as well. 

\paragraph{Predicted and executed cost} Given the controls of $\tau$, we define the \emph{executed} trajectory $\tau^f$ by $x_0^f=x_0$ and $x_{t+1}^f=f(x_t^f,u_t)$, and, given a model $\hat f:\X\times\U\to\X$, the model rollout $\tau^{\hat f}$ by $x_0^{\hat f}=x_0$ and $x_{t+1}^{\hat f}=\hat f(x_t^{\hat f},u_t)$. The cost $J_{x_0}(\tau^f;\omega)$ is the cost incurred when the controls of $\tau$ are applied in open loop, and we refer to $J_{x_0}(\tau^f;\omega)-J_{x_0}(\tau;\omega)$ as the \emph{cost gap} of $\tau$. The following result separates the cost discrepancy into a consistency residual, computable using $\hat f$, and a model-rollout error that depends on the unknown true dynamics.

\begin{proposition}[Cost gap]\label{prop:costgap}
Fix $\omega\in\Omega$ and let $\|v\|_\omega^2=v^\top Q(\omega)v$. For any $x_0\in\X$, any $\tau\in\T_H(x_0)$ and any model $\hat f$,
\begin{equation}
\left|\sqrt{J_{x_0}(\tau^f;\omega)}-\sqrt{J_{x_0}(\tau;\omega)}\right|\leq \rho_\omega(\tau)+e_\omega(\tau),
\label{eq:costgap}
\end{equation}
where
$\rho_\omega(\tau)=(\sum_{t=1}^H\|\hat x_t-x_t^{\hat f}\|_\omega^2)^{1/2}$ and
$e_\omega(\tau)=(\sum_{t=1}^H\|x_t^f-x_t^{\hat f}\|_\omega^2)^{1/2}$.
\end{proposition}
The proof is reported in \cref{app:proofs}. The model error $e_\omega(\tau)$ only depends on the controls of $\tau$ and on the accuracy of $\hat f$, while the \emph{consistency residual} $\rho_\omega(\tau)$, the distance between the predicted states and the model rollout of the controls, depends on the generator, and is what we regularize.

\paragraph{Rollout consistency loss} We fit a one-step model $\hat f$ on the transitions contained in $\D$ by least squares, freeze it, and define, for a query $c=(x_0,\omega)$,
\begin{equation}
\mathcal L_{\rm cons}(\theta;c)=\E_{\tau\sim q_\theta(\cdot\mid c)}\left[\sum_{t=1}^H\norm{D_t^{-1}\left(\hat x_t-x_t^{\hat f}\right)}_2^2\right],
\label{eq:cons}
\end{equation}
where $D_t=\diag(\sigma_t)$, and $\sigma_t$ contains the standard deviations of the state offsets $x_t-x_0$ in $\D$, so that the residual is expressed in the normalized units of the drift loss. Since $\|v\|_\omega^2\leq \kappa_t(\omega)\|D_t^{-1}v\|_2^2$, with $\kappa_t(\omega)$ the largest eigenvalue of $D_tQ(\omega)D_t$, the summand in \cref{eq:cons} upper bounds $\rho_\omega(\tau)^2$ up to the constant $\max_t\kappa_t(\omega)$. The training loss for a query $c$ is then
\begin{equation}
\mathcal L(\theta;c)=\mathcal L_{\rm drift}(\theta;c)+\lambda\,\mathcal L_{\rm cons}(\theta;c),
\label{eq:total_loss}
\end{equation}
with consistency weight $\lambda\geq 0$, where the gradient flows through the predicted states and, via the frozen model $\hat f$, through the controls. We conclude with three remarks.
\emph{(i) Rollout versus one-step consistency.} Penalizing the one-step residuals $r_t=\hat x_{t+1}-\hat f(\hat x_t,u_t)$ would not suffice: if $\hat f(\cdot,u)$ is $L$-Lipschitz, an induction gives $\|\hat x_t-x_t^{\hat f}\|\leq \sum_{k=0}^{t-1}L^{t-1-k}\|r_k\|$, so that on open-loop unstable systems ($L>1$) small one-step residuals are compatible with a rollout residual growing as $L^{H}$.
\emph{(ii) The consistency loss does not select.} \Cref{eq:cons} vanishes on every trajectory that the model can reproduce, whatever its cost: the selection of low-cost plans is still performed by the tilted drift field.
\emph{(iii) The target is preserved when the model is exact.} If $\hat f=f$ on the support of $p_0(\cdot\mid x_0)$, every trajectory in the support of $p_\beta\ll p_0$ has zero residual, so that $p_\beta$ remains an equilibrium of \cref{eq:total_loss} and \cref{thm:variational} still characterizes the target; otherwise, the resulting bias is controlled by $e_\omega$.

\subsection{Full algorithm}\label{sec:algorithm}

We now describe how Drifting RHC is trained in practice and how it is used at test time.

\paragraph{Training phase}
The training objective averages the conditional loss \cref{eq:total_loss} over random planning queries $c=(x_0,\omega)$, with $x_0\sim\hat\rho_0$, the empirical distribution of the initial states $x_i$ of the trajectories in $\D$, and, independently, $\omega\sim\mathrm{Unif}(\Omega)$:
\begin{equation}
\mathcal L_{\mathrm{train}}(\theta)
=
\mathbb E_{x_0\sim \hat\rho_0, \,\omega\sim \mathrm{Unif}(\Omega)}
\left[
\mathcal L(\theta;(x_0,\omega))
\right].
\label{eq:train_objective}
\end{equation}
\begin{algorithm}[t]
\footnotesize
\caption{Training Drifting RHC}
\label{alg:train}
\begin{algorithmic}[1]
\REQUIRE offline dataset $\D=\{\tau_i\}_{i=1}^N$; generator $G_\theta$; neighborhood size $K$; negative batch size $M$; query distribution $\hat\rho_0\times \mathrm{Unif}(\Omega)$; inverse-temperature schedule $\beta$; consistency weight $\lambda$
\STATE Fit a one-step model $\hat f$ by minimizing $\sum\|x_{t+1}-\hat f(x_t,u_t)\|_2^2$ over the transitions in $\D$, and freeze it.
\FOR{each stochastic gradient step}
    \STATE Sample a  batch of queries $c_b=(x_0^{(b)},\omega^{(b)})$, $b=1,\ldots,B$, with
    $
    x_0^{(b)}\sim \hat\rho_0$ and $
    \omega^{(b)}\sim \mathrm{Unif}(\Omega)$.
    \FOR{$b=1,\dots,B$}
        \STATE Construct positive batch ${\cal B}_b^+=(\tau_i^+)_{i=1}^K$ from $\hat p_0(\cdot\mid x_0^{(b)})$ and compute the normalized costs $\bar J_i(\omega^{(b)})$ for each $\tau_i^+\in{\cal B}_b^+$ using \cref{eq:logz}.
        \STATE Sample $\epsilon_1,\ldots,\epsilon_M\sim {\cal N}(0,I)$ and construct a negative batch
        $
        {\cal B}_b^-=(\tau_1^-,\ldots,\tau_M^-),\;
        \tau_j^- = G_\theta(\epsilon_j,c_b)
        $.
        \FOR{$j=1,\dots,M$}
            \STATE Compute the  fields $\widehat V_{\hat p_0}^+(\tau_j^-;\beta,c_b),\widehat V_{q_\theta}^-(\tau_j^-;c_b)$  using \cref{eq:weights_pos,eq:weights_neg} and set
            $
            \widehat V_j
            =
            \widehat V_{\hat p_0}^+(\tau_j^-;\beta,c_b)
            -
            \widehat V_{q_\theta}^-(\tau_j^-;c_b)
            $
            \STATE Roll out the controls of $\tau_j^-$ through $\hat f$ from $x_0^{(b)}$, and compute the residual $\hat\rho_j=\sum_{t=1}^H\|D_t^{-1}(\hat x_{j,t}-x_{j,t}^{\hat f})\|_2^2$
        \ENDFOR
        \STATE Form the empirical query loss
        \[
        \widehat{\mathcal L}(\theta;c_b)
        =
        \tfrac{1}{2M}\textstyle\sum_{j}
        \|
        \tau_j^- -
        \sg(
        \tau_j^- + \widehat V_j
        )
        \|_2^2
        +\tfrac{\lambda}{M}\sum_{j}\hat\rho_j
        \]
    \ENDFOR
    \STATE update $\theta$ by a gradient step on $\frac{1}{B}\sum_{b=1}^B \widehat{\mathcal L}(\theta;c_b)$
\ENDFOR
\end{algorithmic}
\end{algorithm}
Sampling queries in this way amounts to a form of \emph{meta-training}: rather than learning a generator for a single fixed objective, the model learns an amortized map from queries to low-cost trajectory proposals. We minimize \cref{eq:train_objective} by stochastic gradient descent (\cref{alg:train}): for each query, the retrieved positives are \emph{relabeled} with the sampled cost parameter, $\widetilde J_i(\omega)=J_{x_i}(\tau_i^+;\omega)$, and normalized as in \cref{eq:logz}, while the negatives are sampled from the current generator, and their controls are rolled out through $\hat f$ to evaluate \cref{eq:cons}.

\begin{table*}[t]
\centering
\caption{Mean normalized regret, failure rate [\%] and planning time per episode [ms] (100 held-out episodes). Best regret in bold, best among generative and model-based planners underlined.}
\label{tab:main}
\scriptsize
\setlength{\tabcolsep}{2.2pt}
\begin{tabular}{l rrr rrr rrr rrr rrr}
\toprule
 & \multicolumn{3}{c}{MSD, $H=30$} & \multicolumn{3}{c}{MSD, $H=50$} & \multicolumn{3}{c}{MSD, $H=100$} & \multicolumn{3}{c}{Cart-pole, $H=30$} & \multicolumn{3}{c}{Quadrotor, $H=30$} \\
\cmidrule(lr){2-4}\cmidrule(lr){5-7}\cmidrule(lr){8-10}\cmidrule(lr){11-13}\cmidrule(lr){14-16}
Method & Regret & Fail & Time & Regret & Fail & Time & Regret & Fail & Time & Regret & Fail & Time & Regret & Fail & Time \\
\midrule
\textbf{Drifting RHC (consistency)} & \underline{0.026} & 0.0 & 25.4 & \underline{0.029} & 0.0 & 47.3 & \underline{0.054} & 1.0 & 118 & \textbf{0.064} & 0.0 & 27.8 & \textbf{0.064} & 0.0 & 28.5 \\
Drifting RHC (no consistency) & 0.090 & 0.7 & 25.4 & 0.119 & 0.7 & 47.3 & 0.535 & 4.7 & 118 & 11.54 & 90.4 & 27.8 & 7.918 & 90.4 & 27.8 \\
Drifting Prior & 0.152 & 1.0 & 25.6 & 0.184 & 0.7 & 47.1 & 4.291 & 38.7 & 118 & 16.77 & 95.6 & 27.7 & 11.07 & 98.2 & 28.2 \\
Guided Diffusion & 0.051 & 0.0 & 2102 & 0.070 & 0.0 & 3432 & 613.4 & 98.3 & 6797 & 0.540 & 12.3 & 1822 & 1.039 & 35.3 & 1845 \\
Diffusion & 0.094 & 0.0 & 1539 & 0.441 & 10.0 & 2533 & 1326 & 100.0 & 4997 & 7.548 & 62.3 & 1268 & 6.044 & 93.3 & 1291 \\
NN model + CEM & 0.213 & 4.7 & 514 & 0.361 & 11.7 & 1363 & 1.590 & 27.0 & 5136 & 0.870 & 11.7 & 500 & 0.556 & 12.3 & 510 \\
DeePC & \textbf{0.000} & 0.0 & 72.8 & \textbf{0.001} & 0.0 & 116 & \textbf{0.002} & 0.0 & 226 & 0.165 & 2.0 & 70.9 & 0.165 & 0.0 & 73.0 \\
\midrule
\textit{Reference (LQR / iLQR)} & -- & -- & \textit{13.5} & -- & -- & \textit{36.2} & -- & -- & \textit{144} & -- & -- & \textit{250} & -- & -- & \textit{388} \\
\bottomrule
\end{tabular}
\vspace{-0.3cm}
\end{table*}

\paragraph{Inference  and receding-horizon control}
Once training is complete, the generator is used as a one-step proposal mechanism inside a best-of-$M$ control loop.
Given the current state $x$ and a query cost parameter $\omega$, we sample $M_{\rm plan}$ candidate trajectories $\tau^{(m)}=G_\theta(\epsilon_m,(x,\omega))$, evaluate their cost $C_m=J_x(\tau^{(m)};\omega)$ on their predicted states, select the lowest-cost candidate $m^\star\in\argmin_m C_m$ (tie breaks arbitrarily), and execute only its first control. After the next state is observed, the whole procedure is repeated.
The model $\hat f$ is never used at inference: planning requires a single forward pass of $G_\theta$ on $M_{\rm plan}$ noise samples, followed by the evaluation of their cost.

We now give a best-of-$M$ guarantee for this planner over $T$ steps. For a fixed cost parameter $\omega$ and $\delta>0$, we define the set of $\delta$-optimal trajectories
\[
{\cal A}_\delta(x,\omega)=\left\{\tau\in \T_H(x): J_{x}(\tau;\omega) \leq J_x^\star(\omega)+\delta\right \}
\]
for $c=(x,\omega)$ and $J_x^\star(\omega)\coloneqq {\rm ess\inf}_{\tau \in {\cal T}_H(x)} J_{x}(\tau;\omega)$, where the essential infimum is taken with respect to the underlying trajectory law induced by $(\rho_0,f)$. We also let $\hat \tau_t\in \argmin_{m=1,\dots, M_{\rm plan}} J_{x_t}(\tau_t^{(m)};\omega)$ be the trajectory selected by the planner, with $(\tau_{t}^{(i)})_{i}$ sampled i.i.d. from $q_\theta(\cdot\mid x_t,\omega)$, and let
\[E_{\delta,t}=\{J_{x_t}(\hat \tau_t;\omega)>J_{x_t}^\star(\omega)+\delta\}
\]
be the event  that the planner is $\delta$-suboptimal at timestep $t$.

\begin{theorem}[Best-of-$M$ guarantee]\label{thm:bestofm}
Fix $\omega\in \Omega$, and let $p_\beta$ be as in \cref{eq:tiltmeasure}. Assume there exists $\varepsilon(\omega),\eta(\omega)>0$ such that
    \begin{align*}
    &{\rm ess\ sup}_{x\in \X} d_{\rm TV}(q_\theta(\cdot\mid x,\omega), p_\beta(\cdot\mid x,\omega)) \leq \varepsilon(\omega),\\
    &{\rm ess\ inf}_{x\in \X} p_\beta({\cal A}_\delta(x,\omega) \mid x,\omega) \geq \eta(\omega) >0.
    \end{align*}

    Then
    \begin{equation}
        \mathbb{P}\left(\bigcup_{t=0}^{T-1}  E_{\delta,t}\ \Big |\ \omega\right)\leq T\exp\left(- M_{\rm plan}[\eta(\omega)-\varepsilon(\omega)]_+\right),
    \end{equation}
    where $[x]_+=\max(x,0)$.
\end{theorem}
The proof is reported in \cref{app:proofs}. Here, $\eta$ is the probability mass that $p_\beta$ assigns to $\delta$-optimal trajectories, and $\varepsilon$ measures how closely $q_\theta$ matches $p_\beta$ in total variation: whenever $\eta>\varepsilon$, the probability of a $\delta$-suboptimal decision decays exponentially in $M_{\rm plan}$, so that larger planning budgets and better approximations of $p_\beta$ directly improve the reliability of the closed-loop planner.

\Cref{thm:bestofm} concerns the predicted cost of the selected plan. Combining it with \cref{prop:costgap} yields a guarantee on the cost of its controls when applied to the system.
\begin{corollary}[Executed cost of the selected plan]\label{cor:executed}
Under the assumptions of \cref{thm:bestofm} and deterministic dynamics, with probability at least $1-T\exp(-M_{\rm plan}[\eta(\omega)-\varepsilon(\omega)]_+)$, for every $t\in\{0,\ldots,T-1\}$
\[
\sqrt{J_{x_t}(\hat\tau_t^f;\omega)}\leq \sqrt{J_{x_t}^\star(\omega)+\delta}+\rho_\omega(\hat\tau_t)+e_\omega(\hat\tau_t).
\]
\end{corollary}
\begin{proof}
Apply \cref{prop:costgap} on the complement of $\bigcup_{t} E_{\delta,t}$.
\end{proof}
The guarantee thus holds for the executed controls only up to the residual of the \emph{selected} plan, which \cref{eq:cons} reduces.

\section{Numerical Results}\label{sec:numerical_results}
We evaluate Drifting RHC on three systems of increasing complexity, and study the role of its components through ablations.

\subsection{Experimental Setup}

\paragraph{Systems}
The first benchmark is the mass-spring-damper (MSD) system $\dot p = v$, $\dot v = -\frac{k_s}{m}p - \frac{c}{m}v + \frac{1}{m}u$, discretized exactly under zero-order hold, with $m=k_s=1$, $c=0.2$, $\Delta t=0.05$, initial states in $[-2,2]^2$, and planning horizon equal to the episode length $H\in\{30,50,100\}$.
The second is the cart-pole \cite{barto1983Neuronlike}, $x=(p,\dot p,\vartheta,\dot\vartheta)$, to be stabilized at the open-loop unstable upright equilibrium. The third is a planar quadrotor, $x=(p_x,v_x,p_z,v_z,\vartheta,\dot\vartheta)$, with the two rotor thrusts as inputs, which is underactuated, unstable in pitch, and must reach hover at the origin. Both are integrated with a Runge-Kutta scheme with $\Delta t=0.05$, and $H=30$. The cost is \cref{eq:cost} with diagonal $Q(\omega)$ and $R(\omega)$ sampled log-uniformly (details in \cref{app:implementation}).

\paragraph{Datasets and reference controllers}
Offline trajectories are collected by the reference controller with action noise ($10\%$), a random linear state-feedback controller ($10\%$), and a smooth random open-loop controller ($80\%$), for a total of $50{,}000$ (MSD) or $20{,}000$ trajectories, split $70/15/15\%$ into training, validation and test sets. The reference controller is the finite-horizon LQR with the true dynamics for the MSD system, and iLQR \cite{li2004Iterative} with the true dynamics, re-solved at every step, for the nonlinear systems: a strong reference, but not a certified optimum.

\paragraph{Implemented baselines}
We compare \texttt{Drifting RHC (consistency)} (\cref{alg:train}) with: \texttt{Drifting RHC (no consistency)} ($\lambda=0$); \texttt{Drifting Prior}, conditioned only on $x_0$ and without cost tilt; \texttt{Diffusion}, a DDPM-style \cite{ho2020denoising} generator conditioned on $x_0$; \texttt{Guided Diffusion}, the same model with classifier guidance by the cumulative cost, as in Diffuser \cite{pmlr-v162-janner22a};\footnote{Since the cost is known at test time, its gradient is computed in closed form, without training a classifier.} \texttt{DeePC}, a regularized DeePC \cite{coulson2019data,dorfler2023Bridging}; and \texttt{NN model + CEM}, which plans through the model $\hat f$ of our consistency loss with the cross-entropy method \cite{rubinstein1999Crossentropy,chua2018Deep}. The guidance gradient is normalized, and the guidance scale and DeePC regularization are tuned per setting. All generative planners use the same residual MLP backbone, $M_{\rm plan}=64$ candidates, and the final checkpoint of training (\cref{app:implementation}).\footnote{The code is available at \url{github.com/danielefoffano/Receding-Horizon-Control-via-Drifting-Models}}

\paragraph{Evaluation protocol}
All methods are evaluated on the same $100$ held-out episodes per setting, re-planning at every step. With $J$ and $J_{\rm ref}$ the closed-loop costs of a method and of the reference, we report the mean normalized regret $(J-J_{\rm ref})/J_{\rm ref}$, the failure rate ($\%$ of episodes with $J>2J_{\rm ref}$), and the planning time per episode on an NVIDIA RTX 3080 GPU, over $3$ to $5$ random seeds.

\subsection{Results}

\paragraph{Linear system}
On the MSD system (\cref{tab:main}), DeePC recovers the LQR controller almost exactly, as expected for a noiseless linear system \cite{willems2005note}. Among all other methods, Drifting RHC attains the lowest regret at every horizon, with a mean cost within $3\%$ of the LQR controller, while being $2$ to $3$ times faster than DeePC, and even faster than the LQR reference at $H=100$, which re-solves an $H$-step Riccati recursion at every step. The consistency loss reduces the regret by a factor of $3.5$, $4.2$ and $9.9$ at $H=30$, $50$ and $100$. Guided diffusion is competitive at short horizons, but more than $70$ times slower, and both diffusion baselines fail at $H=100$.

\paragraph{Nonlinear systems}
On the cart-pole and the quadrotor, Drifting RHC attains the lowest regret of all methods, $2.6$ times lower than DeePC, with no failed episode and a mean cost about $3\%$ and $6\%$ above the iLQR reference, while being $9$ to $14$, $2.6$, $18$ and $65$ times faster than iLQR, DeePC, NN model + CEM and guided diffusion, respectively. Planning online through the learned model remains far from the reference, although the same model is used offline by our consistency loss, without which Drifting RHC fails on $90\%$ of the episodes, with a regret $180$ and $124$ times larger. Across all settings, Drifting RHC and DeePC are the only methods whose failure rate never exceeds $2\%$. Closed-loop trajectories are shown in \cref{app:results}.

\paragraph{Ablations} \Cref{app:ablations} presents results on ablations for the consistency loss and cost tilt.
\section{Conclusions}\label{sec:conclusions}

We introduced Drifting RHC, an offline trajectory optimization method combining one-step drifting generative models with receding-horizon. Optimal planning is achieved by tilting the learned trajectory distribution toward low-cost trajectories while remaining close to offline data. We further showed that best-of-$M$ selection is only as reliable as the consistency of the generated plans with the dynamics, which a rollout-consistency loss enforces at training time only. Our experiments show that Drifting RHC attains a cost within a few percent of the reference controller, outperforming DeePC, model-based and diffusion-based planners on nonlinear systems while planning faster than all of them. Future work should focus on systems with more complex dynamics.

\bibliographystyle{IEEEtran}
\bibliography{references}

@article{ho2020denoising,
  title={Denoising diffusion probabilistic models},
  author={Ho, Jonathan and Jain, Ajay and Abbeel, Pieter},
  journal={Advances in neural information processing systems},
  volume={33},
  pages={6840--6851},
  year={2020}
}

@misc{deng2026Generative,
    title = {Generative {Modeling} via {Drifting}},
    doi = {10.48550/arXiv.2602.04770},
    urldate = {2026-03-31},
    publisher = {arXiv},
    author = {Deng, Mingyang and Li, He and Li, Tianhong and Du, Yilun and He, Kaiming},
    month = feb,
    year = {2026},
    note = {arXiv:2602.04770 [cs]},
}

@misc{foffano2025Adversariala,
    title = {Adversarial {Diffusion} for {Robust} {Reinforcement} {Learning}},
    doi = {10.48550/arXiv.2509.23846},
    urldate = {2026-03-31},
    publisher = {arXiv},
    author = {Foffano, Daniele and Russo, Alessio and Proutiere, Alexandre},
    month = dec,
    year = {2025},
    note = {arXiv:2509.23846 [cs]},
}

@inproceedings{pmlr-v162-janner22a,
    series = {Proceedings of machine learning research},
    title = {Planning with diffusion for flexible behavior synthesis},
    volume = {162},
    booktitle = {Proceedings of the 39th international conference on machine learning},
    publisher = {PMLR},
    author = {Janner, Michael and Du, Yilun and Tenenbaum, Joshua and Levine, Sergey},
    editor = {Chaudhuri, Kamalika and Jegelka, Stefanie and Song, Le and Szepesvari, Csaba and Niu, Gang and Sabato, Sivan},
    month = jul,
    year = {2022},
    pages = {9902--9915},
}

@inproceedings{chen2021decision,
    title = {Decision transformer: {Reinforcement} learning via sequence modeling},
    volume = {34},
    booktitle = {Advances in neural information processing systems},
    author = {Chen, Lili and Lu, Kevin and Rajeswaran, Aravind and Lee, Kimin and Grover, Aditya and Laskin, Misha and Abbeel, Pieter and Srinivas, Aravind and Mordatch, Igor},
    year = {2021},
    pages = {15084--15097},
}

@inproceedings{NEURIPS2021_099fe6b0,
    title = {Offline reinforcement learning as one big sequence modeling problem},
    volume = {34},
    booktitle = {Advances in neural information processing systems},
    publisher = {Curran Associates, Inc.},
    author = {Janner, Michael and Li, Qiyang and Levine, Sergey},
    editor = {Ranzato, M. and Beygelzimer, A. and Dauphin, Y. and Liang, P.S. and Vaughan, J. Wortman},
    year = {2021},
    pages = {1273--1286},
}

@misc{ajay2023Conditional,
    title = {Is {Conditional} {Generative} {Modeling} all you need for {Decision}-{Making}?},
    doi = {10.48550/arXiv.2211.15657},
    urldate = {2026-03-31},
    publisher = {arXiv},
    author = {Ajay, Anurag and Du, Yilun and Gupta, Abhi and Tenenbaum, Joshua and Jaakkola, Tommi and Agrawal, Pulkit},
    month = jul,
    year = {2023},
    note = {arXiv:2211.15657 [cs]},
}

@misc{wang2023Diffusion,
    title = {Diffusion {Policies} as an {Expressive} {Policy} {Class} for {Offline} {Reinforcement} {Learning}},
    doi = {10.48550/arXiv.2208.06193},
    urldate = {2026-03-31},
    publisher = {arXiv},
    author = {Wang, Zhendong and Hunt, Jonathan J. and Zhou, Mingyuan},
    month = aug,
    year = {2023},
    note = {arXiv:2208.06193 [cs]},
}

@misc{chi2024Diffusion,
    title = {Diffusion {Policy}: {Visuomotor} {Policy} {Learning} via {Action} {Diffusion}},
    shorttitle = {Diffusion {Policy}},
    doi = {10.48550/arXiv.2303.04137},
    urldate = {2026-03-31},
    publisher = {arXiv},
    author = {Chi, Cheng and Xu, Zhenjia and Feng, Siyuan and Cousineau, Eric and Du, Yilun and Burchfiel, Benjamin and Tedrake, Russ and Song, Shuran},
    month = mar,
    year = {2024},
    note = {arXiv:2303.04137 [cs]},
}

@misc{sacks2022Learning,
    title = {Learning {Sampling} {Distributions} for {Model} {Predictive} {Control}},
    doi = {10.48550/arXiv.2212.02587},
    urldate = {2026-03-31},
    publisher = {arXiv},
    author = {Sacks, Jacob and Boots, Byron},
    month = dec,
    year = {2022},
    note = {arXiv:2212.02587 [cs]},
}

@misc{argenson2021ModelBased,
    title = {Model-{Based} {Offline} {Planning}},
    doi = {10.48550/arXiv.2008.05556},
    urldate = {2026-03-31},
    publisher = {arXiv},
    author = {Argenson, Arthur and Dulac-Arnold, Gabriel},
    month = mar,
    year = {2021},
    note = {arXiv:2008.05556 [cs]},
}

@misc{zhan2022ModelBased,
    title = {Model-{Based} {Offline} {Planning} with {Trajectory} {Pruning}},
    doi = {10.48550/arXiv.2105.07351},
    urldate = {2026-03-31},
    publisher = {arXiv},
    author = {Zhan, Xianyuan and Zhu, Xiangyu and Xu, Haoran},
    month = apr,
    year = {2022},
    note = {arXiv:2105.07351 [cs]},
}

@misc{levine2018Reinforcement,
    title = {Reinforcement {Learning} and {Control} as {Probabilistic} {Inference}: {Tutorial} and {Review}},
    shorttitle = {Reinforcement {Learning} and {Control} as {Probabilistic} {Inference}},
    doi = {10.48550/arXiv.1805.00909},
    urldate = {2026-03-31},
    publisher = {arXiv},
    author = {Levine, Sergey},
    month = may,
    year = {2018},
    note = {arXiv:1805.00909 [cs]},
}

@article{willems2005note,
    title = {A note on persistency of excitation},
    volume = {54},
    number = {4},
    journal = {Systems \& Control Letters},
    publisher = {Elsevier},
    author = {Willems, Jan C and Rapisarda, Paolo and Markovsky, Ivan and De Moor, Bart LM},
    year = {2005},
    pages = {325--329},
}

@inproceedings{tang2022data,
    title = {Data-driven control: {Overview} and perspectives},
    booktitle = {2022 {American} control conference ({ACC})},
    publisher = {IEEE},
    author = {Tang, Wentao and Daoutidis, Prodromos},
    year = {2022},
    pages = {1048--1064},
}

@incollection{ljung1998system,
    title = {System identification},
    booktitle = {Signal analysis and prediction},
    publisher = {Springer},
    author = {Ljung, Lennart},
    year = {1998},
    pages = {163--173},
}

@inproceedings{russo2023Tubebased,
    title = {Tube-based zonotopic data-driven predictive control},
    doi = {10.23919/ACC55779.2023.10156056},
    booktitle = {2023 american control conference ({ACC})},
    author = {Russo, Alessio and Proutiere, Alexandre},
    year = {2023},
    pages = {3845--3851},
}

@book{sutton2018reinforcement,
    title = {Reinforcement learning: {An} introduction},
    publisher = {MIT press},
    author = {Sutton, Richard S and Barto, Andrew G},
    year = {2018},
}

@article{polydoros2017survey,
    title = {Survey of model-based reinforcement learning: {Applications} on robotics},
    volume = {86},
    number = {2},
    journal = {Journal of Intelligent \& Robotic Systems},
    publisher = {Springer},
    author = {Polydoros, Athanasios S and Nalpantidis, Lazaros},
    year = {2017},
    pages = {153--173},
}

@inproceedings{janner2019trust,
    title = {When to trust your model: {Model}-based policy optimization},
    volume = {32},
    booktitle = {Advances in neural information processing systems},
    author = {Janner, Michael and Fu, Justin and Zhang, Marvin and Levine, Sergey},
    year = {2019},
}

@article{prudencio2023survey,
    title = {A survey on offline reinforcement learning: {Taxonomy}, review, and open problems},
    volume = {35},
    number = {8},
    journal = {IEEE transactions on neural networks and learning systems},
    publisher = {IEEE},
    author = {Prudencio, Rafael Figueiredo and Maximo, Marcos ROA and Colombini, Esther Luna},
    year = {2023},
    pages = {10237--10257},
}

@inproceedings{yu2020mopo,
    title = {Mopo: {Model}-based offline policy optimization},
    volume = {33},
    booktitle = {Advances in neural information processing systems},
    author = {Yu, Tianhe and Thomas, Garrett and Yu, Lantao and Ermon, Stefano and Zou, James Y and Levine, Sergey and Finn, Chelsea and Ma, Tengyu},
    year = {2020},
    pages = {14129--14142},
}

@inproceedings{kidambi2020morel,
    title = {Morel: {Model}-based offline reinforcement learning},
    volume = {33},
    booktitle = {Advances in neural information processing systems},
    author = {Kidambi, Rahul and Rajeswaran, Aravind and Netrapalli, Praneeth and Joachims, Thorsten},
    year = {2020},
    pages = {21810--21823},
}

@article{emmons2021rvs,
    title = {Rvs: {What} is essential for offline rl via supervised learning?},
    journal = {arXiv preprint arXiv:2112.10751},
    author = {Emmons, Scott and Eysenbach, Benjamin and Kostrikov, Ilya and Levine, Sergey},
    year = {2021},
}

@inproceedings{10155796,
    title = {Self-tuning tube-based model predictive control},
    doi = {10.23919/ACC55779.2023.10155796},
    booktitle = {2023 american control conference ({ACC})},
    author = {Tranos, Damianos and Russo, Alessio and Proutiere, Alexandre},
    year = {2023},
    pages = {3626--3632},
}

@article{celestini2024transformer,
    title = {Transformer-based model predictive control: {Trajectory} optimization via sequence modeling},
    volume = {9},
    number = {11},
    journal = {IEEE Robotics and Automation Letters},
    publisher = {IEEE},
    author = {Celestini, Davide and Gammelli, Daniele and Guffanti, Tommaso and D'Amico, Simone and Capello, Elisa and Pavone, Marco},
    year = {2024},
    pages = {9820--9827},
}

@inproceedings{coulson2019data,
    title = {Data-enabled predictive control: {In} the shallows of the {DeePC}},
    booktitle = {2019 18th {European} control conference ({ECC})},
    publisher = {IEEE},
    author = {Coulson, Jeremy and Lygeros, John and Dörfler, Florian},
    year = {2019},
    pages = {307--312},
}

@article{donsker1975Asymptotic,
    title = {Asymptotic evaluation of certain {Markov} process expectations for large time, {I}},
    volume = {28},
    number = {1},
    journal = {Communications on Pure and Applied Mathematics},
    author = {Donsker, Monroe D. and Varadhan, S. R. Srinivasa},
    year = {1975},
    pages = {1--47},
}

@article{rubinstein1999Crossentropy,
    title = {The cross-entropy method for combinatorial and continuous optimization},
    volume = {1},
    number = {2},
    journal = {Methodology and Computing in Applied Probability},
    author = {Rubinstein, Reuven},
    year = {1999},
    pages = {127--190},
}

@inproceedings{chua2018Deep,
    title = {Deep reinforcement learning in a handful of trials using probabilistic dynamics models},
    volume = {31},
    booktitle = {Advances in neural information processing systems},
    author = {Chua, Kurtland and Calandra, Roberto and McAllister, Rowan and Levine, Sergey},
    year = {2018},
}

@inproceedings{hafner2020Dream,
    title = {Dream to control: {Learning} behaviors by latent imagination},
    booktitle = {International conference on learning representations ({ICLR})},
    author = {Hafner, Danijar and Lillicrap, Timothy and Ba, Jimmy and Norouzi, Mohammad},
    year = {2020},
}

@article{dorfler2023Bridging,
    title = {Bridging direct and indirect data-driven control formulations via regularizations and relaxations},
    volume = {68},
    number = {2},
    journal = {IEEE Transactions on Automatic Control},
    author = {D{\"o}rfler, Florian and Coulson, Jeremy and Markovsky, Ivan},
    year = {2023},
    pages = {883--897},
}

@article{zhou2025Diffusion,
    title = {Diffusion model predictive control},
    journal = {Transactions on Machine Learning Research},
    author = {Zhou, Guangyao and Swaminathan, Sivaramakrishnan and Raju, Rajkumar Vasudeva and Guntupalli, J. Swaroop and Lehrach, Wolfgang and Ortiz, Joseph and Dedieu, Antoine and L{\'a}zaro-Gredilla, Miguel and Murphy, Kevin},
    year = {2025},
}

@inproceedings{li2004Iterative,
    title = {Iterative linear quadratic regulator design for nonlinear biological movement systems},
    booktitle = {Proceedings of the 1st international conference on informatics in control, automation and robotics ({ICINCO})},
    author = {Li, Weiwei and Todorov, Emanuel},
    year = {2004},
    pages = {222--229},
}

@article{barto1983Neuronlike,
    title = {Neuronlike adaptive elements that can solve difficult learning control problems},
    volume = {SMC-13},
    number = {5},
    journal = {IEEE Transactions on Systems, Man, and Cybernetics},
    author = {Barto, Andrew G. and Sutton, Richard S. and Anderson, Charles W.},
    year = {1983},
    pages = {834--846},
}

\newpage
\useRomanappendicesfalse
\appendices

\section{Proofs}\label[appendix]{app:proofs}

\begin{proof}[Proof of \cref{thm:variational}]
For any $p\ll p_0(\cdot\mid x_0)$,
\[
\log \frac{dp}{dp_\beta}
=
\log \frac{dp}{dp_0} + \beta J_{x_0}(\tau;\omega) + \log Z_\beta(c).
\]
Taking expectation with respect to $p$ yields
\[
\KL(p\|p_\beta)
=
\KL(p\|p_0)+\beta\E_p[J_{x_0}(\tau;\omega)] + \log Z_\beta(c).
\]
Rearranging gives
\begin{align*}
&\E_p[J_{x_0}(\tau;\omega)] + \tfrac{1}{\beta}\KL(p\|p_0)\\
&\qquad= \tfrac{1}{\beta}\KL(p\|p_\beta) - \tfrac{1}{\beta}\log Z_\beta(c).
\end{align*}
The second term on the r.h.s. is constant in $p$, while the first is non-negative and vanishes only when $p=p_\beta$.
\end{proof}

\begin{proof}[Proof of \cref{lem:weighted_positives_tilt}]
By definition of $p_\beta$, for any integrable $h$ we have
\begin{align*}
\mathbb{E}_{\tau'\sim p_\beta}[h(\tau')]
&=\int h(\tau')\frac{w_\beta(\tau';c)}{Z_\beta(c)}p_0({\rm d}\tau'\mid x_0),\\
&=\frac{1}{Z_\beta(c)}\mathbb{E}_{\tau'\sim p_0(\cdot\mid x_0)}[w_\beta(\tau';c) h(\tau')].
\end{align*}
Apply this identity in $V_{p_\beta}^+(\tau;x_0)$ with $h(\tau')=k(\tau,\tau')(\tau'-\tau)$ for the numerator and
$h(\tau')=k(\tau,\tau')$ for the denominator. The factor $1/Z_\beta(c)$ cancels between numerator and denominator, yielding the claim.
\end{proof}

\begin{proof}[Proof of \cref{prop:costgap}]
Let $\Phi_\omega(\tau)$ be the vector that stacks $Q(\omega)^{1/2}\hat x_t$ for $t=0,\ldots,H$ (with $\hat x_0=x_0$) and $R(\omega)^{1/2}u_t$ for $t=0,\ldots,H-1$, so that $J_{x_0}(\tau;\omega)=\|\Phi_\omega(\tau)\|_2^2$ by \cref{eq:cost}. Since $\tau$ and $\tau^f$ share the initial state and the controls, the only non-zero entries of $\Phi_\omega(\tau^f)-\Phi_\omega(\tau)$ are $Q(\omega)^{1/2}(x_t^f-\hat x_t)$, $t=1,\ldots,H$. The reverse triangle inequality gives
\[
\left|\|\Phi_\omega(\tau^f)\|_2-\|\Phi_\omega(\tau)\|_2\right|\leq \Big(\textstyle\sum_{t=1}^H\|x_t^f-\hat x_t\|_\omega^2\Big)^{1/2}.
\]
Writing $x_t^f-\hat x_t=(x_t^f-x_t^{\hat f})+(x_t^{\hat f}-\hat x_t)$ and applying Minkowski's inequality concludes the proof.
\end{proof}

\begin{proof}[Proof of \cref{thm:bestofm}]
    Let $\{{\cal F}_t\}_{t\geq 0}$ be the filtration  of the history up to  step $t$, and let $c_t=(x_t,\omega)$ be the query at step $t$. Fix a timestep $t\geq 0$. Given the current query $c_t$, the candidates $(\tau_t^{(i)})_{i=1}^{M_{\rm plan}}$ are sampled i.i.d. from $q_\theta(\cdot\mid c_t)$. Hence,
    \begin{align*}
    \mathbb{P}(E_{\delta,t} \mid {\cal F}_t,\omega)&=\prod_{m=1}^{M_{\rm plan}} \mathbb{P}\left(\tau_t^{(m)}\notin {\cal A}_\delta(x_t,\omega)\mid {\cal F}_t,\omega\right),\\
    &=\left[1- q_\theta({\cal A}_\delta(x_t,\omega)\mid c_t)\right]^{M_{\rm plan}}.
    \end{align*}
    Now, since $d_{\rm TV}(q_\theta(\cdot\mid c_t),p_\beta(\cdot\mid c_t))=\sup_{\text{meas. } A} |q_\theta(A\mid c_t)- p_{\beta}(A\mid c_t)|$, it follows that for any measurable event $A$ we have
    \[
    q_\theta(A\mid c_t)\geq [p_{\beta}(A\mid c_t)-d_{\rm TV}(q_\theta(\cdot\mid c_t),p_\beta(\cdot\mid c_t))]_+,
    \]
    where $[x]_+=\max(x,0)$.
    Applying this result with  $A={\cal A}_\delta(x_t,\omega)$, and letting $\nu_t=[p_{\beta}(A\mid c_t)-d_{\rm TV}(q_\theta(\cdot\mid c_t),p_\beta(\cdot\mid c_t))]_+\in[0,1]$, we get
    \begin{align*}
    \mathbb{P}(E_{\delta,t} \mid {\cal F}_t,\omega)
    &\leq \left[1- \nu_t\right]^{M_{\rm plan}}\leq \exp\left\{-M_{\rm plan}\, \nu_t \right\},\\
    &\leq \exp\left\{-M_{\rm plan} \left[\eta(\omega)-\varepsilon(\omega)\right]_+ \right\},
    \end{align*}
    where we used $1-\nu\leq e^{-\nu}$ and the monotonicity of $[\cdot]_+$.
    The conclusion follows from a tower rule argument and a union bound over timesteps.
\end{proof}

\section{Implementation Details}\label[appendix]{app:implementation}

\paragraph{Systems} The cart-pole uses the dynamics of \cite{barto1983Neuronlike} with a cart mass of $1.0$, a pole mass of $0.1$, a pole half-length of $0.5$ and $g=9.81$, without input saturation. Initial states are sampled uniformly with $|p|\leq 1$, $|\dot p|\leq 0.5$, $|\vartheta|\leq 0.2$, $|\dot\vartheta|\leq 0.5$. For the planar quadrotor, with inputs $u=(\Delta u_L,\Delta u_R)$ and total thrust $F=mg+\Delta u_L+\Delta u_R$, the dynamics are
\begin{align*}
&\ddot p_x=-\tfrac{F}{m}\sin\vartheta,\quad \ddot p_z=\tfrac{F}{m}\cos\vartheta-g,\\
&\ddot\vartheta=\tfrac{\ell}{I}(\Delta u_L-\Delta u_R),
\end{align*}
with $m=1.0$, $I=0.05$, $\ell=0.25$, $g=9.81$, and each input saturated at $\pm10$. Initial states are sampled uniformly with $|p_x|,|p_z|\leq 2$, $|v_x|,|v_z|\leq 1.5$, $|\vartheta|\leq 0.5$, $|\dot\vartheta|\leq 2$. In both systems the action is held constant over each step of the Runge-Kutta integration. For iLQR, we use at most $30$ iterations per step.

\paragraph{Cost parameters} In all systems, $Q(\omega)$ and $R(\omega)$ are diagonal, and each diagonal entry is sampled log-uniformly in the interval reported in \cref{tab:costranges}, both for the offline dataset and for the test episodes.

\begin{table}[h]
\centering
\caption{Sampling intervals of the diagonal entries of $Q(\omega)$ and $R(\omega)$ (log-uniform).}
\label{tab:costranges}
\footnotesize
\setlength{\tabcolsep}{3pt}
\begin{tabular}{l l l}
\toprule
System & $Q(\omega)$ & $R(\omega)$ \\
\midrule
MSD & $p:[0.5,20]$, $v:[0.05,5]$ & $u:[0.01,1]$ \\
Cart-pole & $p:[0.5,5]$, $\dot p:[0.05,1]$, & $u:[0.01,0.5]$ \\
 & $\vartheta:[5,50]$, $\dot\vartheta:[0.05,2]$ & \\
Quadrotor & $p_x,p_z:[0.5,5]$, $v_x,v_z:[0.05,1]$, & both: $[0.01,0.5]$ \\
 & $\vartheta:[2,20]$, $\dot\vartheta:[0.05,1]$ & \\
\bottomrule
\end{tabular}
\end{table}

\paragraph{Datasets} The reference controller used for data collection is the one used for evaluation (LQR or iLQR), with Gaussian action noise of standard deviation $0.05$ (MSD, cart-pole) or $0.1$ (quadrotor). The linear state-feedback controller is $u=-Kx$ plus Gaussian noise, with gains drawn uniformly for each trajectory, alternating position-like entries in $[0.5,4]$ and rate-like entries in $[0.2,2.5]$ (MSD, cart-pole), or in $[0.2,1.5]$ and $[0.3,1.5]$ (quadrotor). The smooth random open-loop controller is a first-order autoregressive process $u_t=0.85\,u_{t-1}+0.15\,\xi_t$, with $\xi_t$ Gaussian with standard deviation $1$ (MSD), $5$ (cart-pole) or $4$ (quadrotor).

\paragraph{Drifting generators} All drifting generators share the same architecture: a residual MLP that takes as input the noise ($d_\epsilon=32$), the initial state and, for the cost-aware variants, the cost parameter, followed by a linear layer of width $256$, four residual blocks with layer normalization and SiLU activations, and two linear heads that output the controls and the state offsets $x_t-x_0$. Trajectories are normalized coordinate-wise with the mean and standard deviation of the training set. We use $K=256$ positives retrieved with a context kernel of temperature $2$, $M=32$ negatives, $128$ queries per batch, and AdamW with learning rate $4\cdot10^{-4}$, weight decay $10^{-4}$ and gradient clipping at $1$. The drift kernel is Gaussian with temperature $T=5$; at $H=50$ and $H=100$ we average the kernels with temperatures $\{1,5\}$ and $\{5,20\}$, respectively. The inverse temperature is increased linearly from $0$ to $\beta_{\max}$ over the first $40\%$ of training, and the normalized costs \cref{eq:logz} are clipped at $\bar c=5$. The setting-specific values are reported in \cref{tab:hyper}. The Drifting Prior uses the same architecture and training, without cost conditioning and with $\beta=0$.

\paragraph{Dynamics model} The model $\hat f$ is a residual MLP $\hat f(x,u)=x+g(x,u)$, where $g$ has three hidden layers of width $256$ with SiLU activations. It is trained on the transitions of the training set by minimizing the mean squared one-step prediction error with AdamW (learning rate $10^{-3}$, cosine schedule, $200$ epochs, batch size $1024$), and the model with the lowest validation error is kept. The validation mean squared error is of order $10^{-10}$ on the MSD system, $4.3\cdot 10^{-7}$ on the cart-pole and $1.0\cdot 10^{-4}$ on the quadrotor. The same model is used by the consistency loss and by the NN model + CEM baseline.

\paragraph{Diffusion baselines} The denoiser is a residual MLP of width $256$ with four blocks and a sinusoidal time embedding of dimension $64$, trained as a DDPM with $64$ steps and a linear noise schedule from $10^{-4}$ to $0.02$ for $500$ epochs, and sampled with all $64$ steps. For guided diffusion, the guidance gradient is normalized, the step is capped at $0.5$, and the guidance scale is selected for each setting from $\{0.1,0.25,0.5,1,2\}$ (\cref{tab:hyper}).

\paragraph{DeePC} The data matrix is built from $300$ trajectories of the training set, with an initialization length of one step and a prediction horizon of $H-1$ steps. The soft initial-condition constraint has weight $10^5$. Since the problem has no inequality constraints, it is an unconstrained quadratic program in the combination vector, whose normal matrix is factorized once per episode, so that every step only requires two solves with its Cholesky factor. The weight $\lambda_g$ is selected for each setting (\cref{tab:hyper}).

\paragraph{NN model + CEM} At every step, CEM optimizes the control sequence over the horizon $H$ with $256$ samples, $32$ elites and $5$ iterations, starting from a zero-mean Gaussian with unit standard deviation and smoothing the updates of the mean and standard deviation with coefficient $0.1$. The first control of the best sampled sequence is executed.

\begin{table}[h]
\centering
\caption{Setting-specific hyperparameters.}
\label{tab:hyper}
\footnotesize
\setlength{\tabcolsep}{3pt}
\begin{tabular}{l ccccc}
\toprule
Setting & $\beta_{\max}$ & $\lambda$ & Epochs (drift) & Guidance scale & $\lambda_g$ \\
\midrule
MSD, $H=30$ & 1 & 30 & 500 & 0.25 & 1 \\
MSD, $H=50$ & 1 & 70 & 500 & 1 & 1 \\
MSD, $H=100$ & 1 & 70 & 500 & 2 & 10 \\
Cart-pole & 3 & 1.2 & 200 & 1 & 100 \\
Quadrotor & 3 & 15 & 200 & 1 & 300 \\
\bottomrule
\end{tabular}
\end{table}

\section{Additional Results}\label[appendix]{app:results}

\paragraph{Detailed results} \Cref{tab:msd,tab:nonlinear} report the mean cost and the standard deviation of the regret across random seeds for all the settings of \cref{tab:main}, and \cref{tab:oracle} the ablation on the fraction of reference-controller trajectories in the dataset discussed in \cref{app:ablations}.

\begin{table}[h]
\centering
\caption{Mass-spring-damper. Mean cost, mean normalized regret (mean $\pm$ s.d.\ over random seeds), failure rate, and planning time per episode. Best regret in bold, best regret among generative and model-based planners underlined.}
\label{tab:msd}
\footnotesize
\setlength{\tabcolsep}{2pt}
\begin{tabular}{l rrrr}
\toprule
Method & Cost & Regret & Fail [\%] & Time [ms] \\
\midrule
\multicolumn{5}{l}{\textit{Horizon 30}} \\[2pt]
\textit{LQR (true dynamics)} & \textit{104.3} & -- & -- & \textit{13.5} \\
DeePC$^{\ddagger}$ & 104.3 & \textbf{0.000} & 0.0 & 72.8 \\
\textbf{Drifting RHC (consistency)} & 106.2 & \underline{0.026\,$\pm$\,0.000} & 0.0 & 25.4 \\
Guided Diffusion & 109.5 & 0.051\,$\pm$\,0.001 & 0.0 & 2102 \\
Drifting RHC (no consistency) & 110.3 & 0.090\,$\pm$\,0.002 & 0.7 & 25.4 \\
Diffusion & 112.4 & 0.094\,$\pm$\,0.008 & 0.0 & 1539 \\
Drifting Prior & 113.4 & 0.152\,$\pm$\,0.001 & 1.0 & 25.6 \\
NN model + CEM & 131.0 & 0.213\,$\pm$\,0.004 & 4.7 & 513.6 \\
\midrule
\multicolumn{5}{l}{\textit{Horizon 50}} \\[2pt]
\textit{LQR (true dynamics)} & \textit{81.3} & -- & -- & \textit{36.2} \\
DeePC$^{\ddagger}$ & 81.3 & \textbf{0.001} & 0.0 & 115.8 \\
\textbf{Drifting RHC (consistency)} & 83.4 & \underline{0.029\,$\pm$\,0.001} & 0.0 & 47.3 \\
Guided Diffusion & 87.2 & 0.070\,$\pm$\,0.002 & 0.0 & 3432 \\
Drifting RHC (no consistency) & 87.8 & 0.119\,$\pm$\,0.005 & 0.7 & 47.3 \\
Drifting Prior & 90.8 & 0.184\,$\pm$\,0.008 & 0.7 & 47.1 \\
NN model + CEM & 116.9 & 0.361\,$\pm$\,0.017 & 11.7 & 1363 \\
Diffusion & 105.5 & 0.441\,$\pm$\,0.010 & 10.0 & 2533 \\
\midrule
\multicolumn{5}{l}{\textit{Horizon 100}} \\[2pt]
\textit{LQR (true dynamics)} & \textit{92.8} & -- & -- & \textit{143.8} \\
DeePC$^{\ddagger}$ & 92.9 & \textbf{0.002} & 0.0 & 226.2 \\
\textbf{Drifting RHC (consistency)} & 95.6 & \underline{0.054\,$\pm$\,0.003} & 1.0 & 117.9 \\
Drifting RHC (no consistency) & 100.2 & 0.535\,$\pm$\,0.031 & 4.7 & 118.1 \\
NN model + CEM & 142.0 & 1.590\,$\pm$\,0.128 & 27.0 & 5136 \\
Drifting Prior & 135.9 & 4.291\,$\pm$\,0.584 & 38.7 & 117.9 \\
Guided Diffusion & 6320 & 613.4\,$\pm$\,169.3 & 98.3 & 6797 \\
Diffusion & 8070 & 1326\,$\pm$\,323 & 100.0 & 4997 \\
\bottomrule
\end{tabular}
\\[2pt]
{\scriptsize $^{\ddagger}$Deterministic given the episode set, hence no error bar.}
\end{table}

\begin{table}[h]
\centering
\caption{Nonlinear systems. Same metrics as in \cref{tab:msd}, with the iLQR reference controller.}
\label{tab:nonlinear}
\footnotesize
\setlength{\tabcolsep}{2pt}
\begin{tabular}{l rrrr}
\toprule
Method & Cost & Regret & Fail [\%] & Time [ms] \\
\midrule
\multicolumn{5}{l}{\textit{Cart-pole, $H=30$}} \\[2pt]
\textit{iLQR (true dynamics)} & \textit{28.1} & -- & -- & \textit{250.3} \\
\textbf{Drifting RHC (consistency)} & 28.9 & \textbf{0.064\,$\pm$\,0.003} & 0.0 & 27.8 \\
DeePC$^{\ddagger}$ & 33.1 & 0.165 & 2.0 & 70.9 \\
Guided Diffusion & 34.6 & 0.540\,$\pm$\,0.033 & 12.3 & 1822 \\
NN model + CEM & 50.9 & 0.870\,$\pm$\,0.555 & 11.7 & 500.3 \\
Diffusion & 96.1 & 7.548\,$\pm$\,5.776 & 62.3 & 1268 \\
Drifting RHC (no consistency) & 243.5 & 11.535\,$\pm$\,0.458 & 90.4 & 27.8 \\
Drifting Prior & 221.8 & 16.772\,$\pm$\,2.226 & 95.6 & 27.7 \\
\midrule
\multicolumn{5}{l}{\textit{Planar quadrotor, $H=30$}} \\[2pt]
\textit{iLQR (true dynamics)} & \textit{97.3} & -- & -- & \textit{388.0} \\
\textbf{Drifting RHC (consistency)} & 103.4 & \textbf{0.064\,$\pm$\,0.001} & 0.0 & 28.5 \\
DeePC$^{\ddagger}$ & 113.8 & 0.165 & 0.0 & 73.0 \\
NN model + CEM & 144.4 & 0.556\,$\pm$\,0.018 & 12.3 & 510.5 \\
Guided Diffusion & 163.5 & 1.039\,$\pm$\,0.073 & 35.3 & 1845 \\
Diffusion & 411.5 & 6.044\,$\pm$\,0.176 & 93.3 & 1291 \\
Drifting RHC (no consistency) & 748.4 & 7.918\,$\pm$\,1.357 & 90.4 & 27.8 \\
Drifting Prior & 645.1 & 11.073\,$\pm$\,1.204 & 98.2 & 28.2 \\
\bottomrule
\end{tabular}
\\[2pt]
{\scriptsize $^{\ddagger}$Deterministic given the episode set, hence no error bar.}
\end{table}

\begin{table}[h]
\centering
\caption{MSD system ($H=30$), median regret ($\pm$ s.e.\ over 3 seeds) as a function of the fraction of reference-controller trajectories in $\D$. Configuration without consistency loss, with z-score cost normalization, $M_{\rm plan}=5$, $200$ training epochs and $200$ evaluation episodes shared by all models.}
\label{tab:oracle}
\footnotesize
\setlength{\tabcolsep}{2.5pt}
\begin{tabular}{l ccc}
\toprule
Reference data & Drifting RHC & $\omega$-cond., no tilt & Drifting Prior \\
\midrule
0\% & \textbf{0.282\,$\pm$\,0.007} & 0.344\,$\pm$\,0.010 & 0.373\,$\pm$\,0.021 \\
1\% & \textbf{0.236\,$\pm$\,0.001} & 0.338\,$\pm$\,0.020 & 0.347\,$\pm$\,0.007 \\
5\% & \textbf{0.220\,$\pm$\,0.004} & 0.350\,$\pm$\,0.014 & 0.341\,$\pm$\,0.001 \\
10\% & \textbf{0.191\,$\pm$\,0.008} & 0.344\,$\pm$\,0.030 & 0.341\,$\pm$\,0.011 \\
25\% & \textbf{0.174\,$\pm$\,0.006} & 0.295\,$\pm$\,0.009 & 0.278\,$\pm$\,0.018 \\
\bottomrule
\end{tabular}
\end{table}

\begin{figure*}[h]
    \centering
    \includegraphics[width=0.92\textwidth]{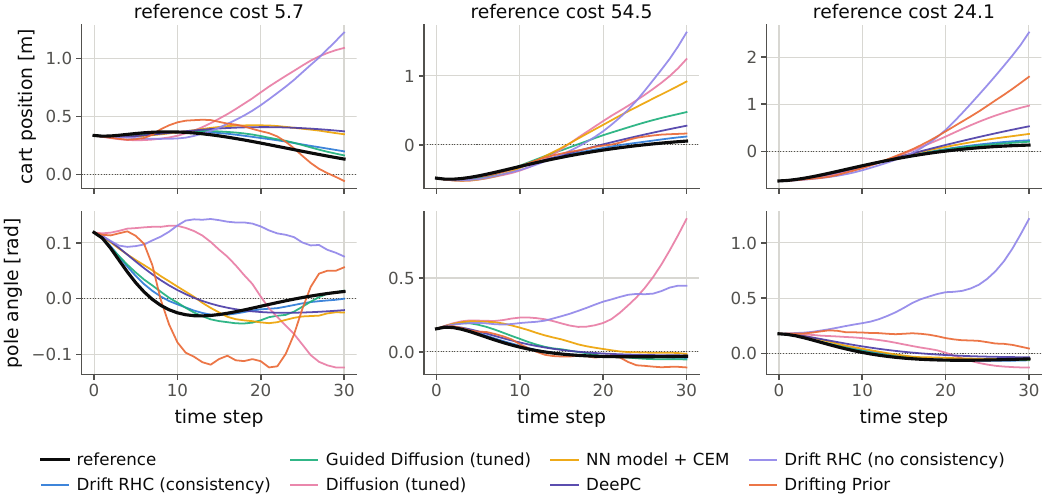}
    \caption{Cart-pole: cart position (top) and pole angle (bottom) of the closed-loop trajectories for three test episodes of different cost, compared with the iLQR reference (black).}
    \label{fig:cartpole}
\end{figure*}

\begin{table}[h]
\centering
\caption{Cart-pole, consistency mechanism: median one-step residual of the selected plan with respect to the true dynamics, and median cost gap $(J(\hat\tau^f)-J(\hat\tau))/J(\hat\tau)$ of the selected plan.}
\label{tab:costgap}
\footnotesize
\setlength{\tabcolsep}{3pt}
\begin{tabular}{l cc}
\toprule
Method & One-step residual & Cost gap \\
\midrule
Drifting RHC (no consistency) & $0.114$ & $+141.3$ \\
\quad + one-step consistency & $0.041$ & $+130.7$ \\
\quad + rollout consistency ($\lambda=10$) & $0.034$ & $+3.28$ \\
\quad + rollout consistency ($\lambda=100$) & $0.035$ & $+0.085$ \\
Guided Diffusion & $0.069$ & $+2.78$ \\
\bottomrule
\end{tabular}
\end{table}

\paragraph{Trajectories and consistency mechanism} \Cref{fig:quadrotor} shows the closed-loop trajectories on the planar quadrotor, and \cref{tab:costgap} the residuals and cost gaps discussed in \cref{app:ablations}. \Cref{fig:cartpole} shows the closed-loop trajectories on the cart-pole. Drifting RHC stabilizes the pole angle along the reference, while the same model without consistency loss keeps the pole away from the upright position and lets the cart drift. Other baselines stabilize the pole but let the cart drift away from the origin, which the cost alone does not reveal.

\paragraph{Sensitivity to the consistency weight} \Cref{tab:lambda} reports the regret of Drifting RHC for different values of $\lambda$, with the same protocol as in \cref{tab:msd,tab:nonlinear}. Thanks to the normalization of the residual in \cref{eq:cons}, the selected values span $\lambda\in[1.2,70]$ across systems. On the MSD system the regret is nearly insensitive to $\lambda$ at $H=30$ and $H=50$, while at $H=100$ larger weights are preferable, since errors in the predicted states accumulate over more steps. On the nonlinear systems, a weight that is too small leaves the predicted states inconsistent, while a weight that is too large lets the consistency term dominate the drift term (cart-pole, $\lambda=4$).

\begin{table}[h]
\centering
\caption{Mean regret ($\pm$ s.d.\ over random seeds) and failure rate of Drifting RHC for different consistency weights. The selected values are in bold.}
\label{tab:lambda}
\footnotesize
\setlength{\tabcolsep}{4pt}
\begin{tabular}{l c c c}
\toprule
System & $\lambda$ & Regret & Fail [\%] \\
\midrule
MSD, $H=30$ & 15 & 0.028\,$\pm$\,0.001 & 0.0 \\
 & \textbf{30} & \textbf{0.026\,$\pm$\,0.000} & 0.0 \\
 & 70 & 0.026\,$\pm$\,0.002 & 0.0 \\
\midrule
MSD, $H=50$ & 15 & 0.031\,$\pm$\,0.001 & 0.0 \\
 & 30 & 0.034\,$\pm$\,0.001 & 0.0 \\
 & \textbf{70} & \textbf{0.029\,$\pm$\,0.001} & 0.0 \\
\midrule
MSD, $H=100$ & 15 & 0.111\,$\pm$\,0.006 & 1.7 \\
 & 30 & 0.096\,$\pm$\,0.006 & 2.0 \\
 & \textbf{70} & \textbf{0.054\,$\pm$\,0.003} & 1.0 \\
\midrule
Cart-pole & 0.4 & 0.125\,$\pm$\,0.008 & 0.2 \\
 & \textbf{1.2} & \textbf{0.064\,$\pm$\,0.003} & 0.0 \\
 & 4 & 0.342\,$\pm$\,0.011 & 5.8 \\
\midrule
Quadrotor & 2 & 0.273\,$\pm$\,0.011 & 3.6 \\
 & 6 & 0.086\,$\pm$\,0.002 & 0.0 \\
 & \textbf{15} & \textbf{0.064\,$\pm$\,0.001} & 0.0 \\
 & 30 & 0.075\,$\pm$\,0.001 & 0.0 \\
\bottomrule
\end{tabular}
\end{table}

\paragraph{Drift and consistency are complementary} \Cref{tab:arms} compares four training objectives: the drift loss alone, the full objective \cref{eq:total_loss}, the consistency loss applied to the behavior prior, and the consistency loss alone, without the drift loss. The consistency loss alone leads to the largest regret in every system: it is satisfied by any trajectory that the model can reproduce, and therefore cannot select low-cost plans. Conversely, the drift loss alone selects low-cost plans whose predicted states are inconsistent. Only the combination of the two attains a low regret on every system.

\begin{table}[h]
\centering
\caption{Mean regret ($M_{\rm plan}=64$, $100$ episodes) of four training objectives.}
\label{tab:arms}
\scriptsize
\setlength{\tabcolsep}{2.5pt}
\begin{tabular}{l cccc}
\toprule
Objective & MSD, $H=30$ & MSD, $H=100$ & Cart-pole & Quadrotor \\
\midrule
Drift only & 0.092 & 0.174 & 13.16 & 6.44 \\
Drift + consistency & \textbf{0.031} & \textbf{0.053} & \textbf{0.27} & \textbf{0.08} \\
Prior + consistency & 0.043 & 0.168 & 3.52 & 0.41 \\
Consistency only & 1.420 & 1.908 & 1768.5 & 3.58 \\
\bottomrule
\end{tabular}
\end{table}

\paragraph{Mass-spring-damper trajectories} \Cref{fig:msd} shows the closed-loop trajectories on the MSD system at $H=100$. Both diffusion planners follow the reference for the first steps, then depart from it with non-smooth velocity profiles, which explains their failure rate in \cref{tab:msd}. The trajectories of Drifting RHC and DeePC overlap with the reference at this scale, while the other drifting planners and NN model + CEM show small deviations from it.

\begin{figure*}[h]
    \centering
    \includegraphics[width=0.92\textwidth]{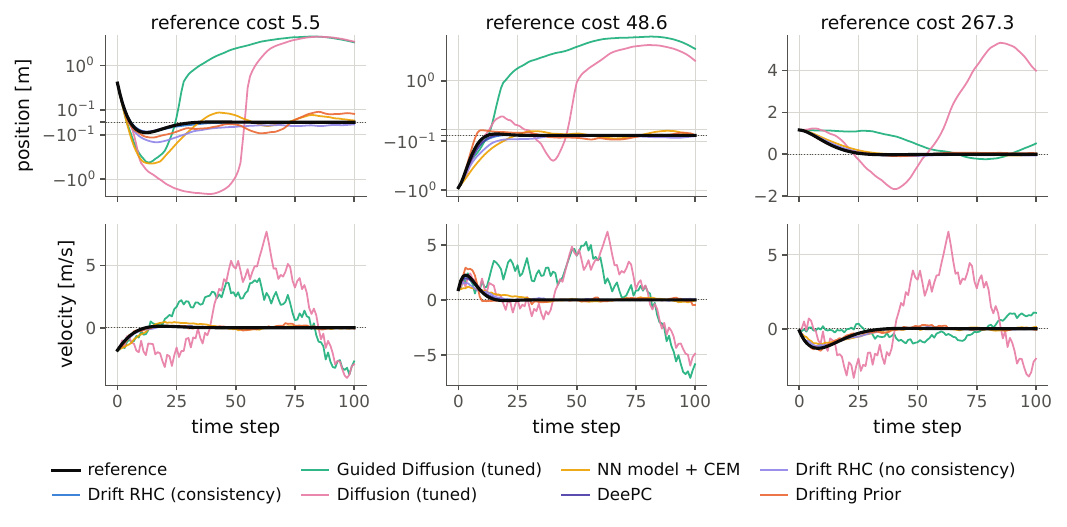}
    \caption{Mass-spring-damper ($H=100$): position (top, symmetric logarithmic scale on the first two episodes) and velocity (bottom) of the closed-loop trajectories for three test episodes of increasing cost, compared with the LQR reference (black).}
    \label{fig:msd}
\end{figure*}

\section{Ablations}\label[appendix]{app:ablations}
On the cart-pole, we measured the cost gap of the selected plan, i.e., the relative difference between the cost of its controls executed in open loop and its predicted cost (\cref{tab:costgap} in \cref{app:results}). Without consistency loss, the median gap is $+141$: best-of-$M$ selection ranks plans that the system cannot follow. One-step consistency leaves it at $+131$, while rollout consistency reduces it to $+3.3$ or less and the median regret from more than $3$ to $0.12$. The consistency loss alone, without drift, leads to a regret $36$ to over $6000$ times larger than the full method (\cref{tab:arms}). These results are expected, as explained in remarks (i) and (ii) in \cref{sec:consistency}. Finally, with $0\%$ to $25\%$ of reference-controller trajectories in the MSD dataset (\cref{tab:oracle}), the tilt improves over the behavior prior at every fraction, while conditioning on $\omega$ without tilt performs as the prior: the gain is due to the tilt, not to near-optimal data.

\end{document}